\pdfoutput=1
\documentclass{article}
\usepackage{iclr2027_conference,times}
\usepackage{amsmath,amssymb,amsthm}
\usepackage{booktabs}
\usepackage{graphicx}
\usepackage{url}
\usepackage{algorithm}
\usepackage{algpseudocode}
\usepackage{xcolor}
\usepackage{etoolbox}
\usepackage{hyperref}
\usepackage{xspace}
\usepackage{multirow}
\usepackage{multicol}
\usepackage{subcaption}
\newtheorem{theorem}{Theorem}
\newtheorem{proposition}[theorem]{Proposition}
\newtheorem{corollary}[theorem]{Corollary}
\newtheorem{lemma}[theorem]{Lemma}

\newcommand{\restatedname}{}
\newtheorem*{restatedthm}{\restatedname}
\newenvironment{restated}[2]
  {\renewcommand{\restatedname}{Theorem~\ref*{#1}}\begin{restatedthm}[#2, restated]}
  {\end{restatedthm}}

\newcommand{\passk}{\mathrm{pass@12}}
\newcommand{\avgk}{\mathrm{avg@12}}
\algrenewcommand\alglinenumber[1]{\scriptsize #1.}
\newcommand{\method}{ARG\xspace}
\setcitestyle{citesep={,},yysep={,}}
\makeatletter
\patchcmd{\@makecaption}{#1: #2}{#1. #2}{}{}
\patchcmd{\@makecaption}{#1: #2}{#1. #2}{}{}
\makeatother

\title{Balancing Reference Guidance and Free Generation in Trajectory Rollouts for Reasoning RL}
\author{
Hanyu Wang$^{1,2}$\thanks{Work done during an internship at Honda Research Institute USA.}\hphantom{$^{*}$}, Nakul Agarwal$^{2}$, Hossein Nourkhiz Mahjoub$^{2}$, Ehsan Moradi Pari$^{2}$, \\
\textbf{Makoto Fukushima$^{3}$, Jinghui Chen$^{1}$, Vaishnav Tadiparthi$^{2}$} \\
$^{1}$The Pennsylvania State University \quad $^{2}$Honda Research Institute USA \\
$^{3}$Honda Research Institute Japan
}

\iclrfinalcopy

\begin{document}
\maketitle
\lhead{Preprint}

\begin{abstract}
A verified reference solution provides a correct trajectory for training a reasoning model. Alternatively, a prefix of the reference can guide the model in generating a trajectory of its own. How much reference guidance should we provide? We study this question through prefix continuation, where the model continues from a reference prefix and keeps the resulting trajectory if it passes verification, falling back to the reference otherwise. Since both procedures produce correct trajectories, we compare their distributions with the ideal distribution, the model's own distribution conditioned on successful verification. For one continuation, we derive the KL divergence in closed form, which, up to a bounded term, decreases with the product of the probability of generating a different correct trajectory and the reference surprisal, the negative log probability of the reference suffix given the prefix. Since a longer prefix tends to raise the former but lowers the latter, continuation success alone does not determine the preferred amount of guidance. From this analysis, we learn a prefix selector shared across training questions from continuation outcomes, without estimating success probabilities or additional generation. The resulting \underline{\textbf{A}}daptive \underline{\textbf{R}}eference \underline{\textbf{G}}uidance (\textbf{ARG}) constructs correct trajectories within a fixed generation budget, and we apply it to all-failure groups in Group Relative Policy Optimization (GRPO). Experiments on Qwen3-4B and Qwen3-8B across five mathematical reasoning benchmarks show that \method achieves the highest aggregate pass@12 among the evaluated methods with competitive average sampled accuracy.
\end{abstract}

\section{Introduction}
\label{sec-introduction}

Reinforcement learning with verifiable rewards (RLVR) \citep{lambert2024tulu} trains Large Language Models (LLMs) using feedback on their own generated responses \citep{guo2025deepseek, team2025kimi, jaech2024openai}. On difficult questions, however, the model may rarely generate a response that reaches the correct answer, leaving little feedback to learn from. When such a question comes with a reference solution written by humans or a stronger model, a straightforward approach is to train on it directly \citep{yan2026learning, fu2026srft, ma2026learning}. However, the reference may differ substantially from the trajectories the model would produce itself, which can make it a poor learning target \citep{setlur2026reuse, qu2026pope}. Since the model learns more easily from familiar responses \citep{ren2024learn}, another line of work uses the reference as guidance so that the model generates correct trajectories by itself. Some methods transform the reference into hints \citep{zhang2026scaf,liao2026self}, which requires additional LLM generation and may introduce bias or errors. Other methods, in contrast, let the model generate a \emph{continuation} from a \emph{reference prefix}, which is an initial segment of the reference. This approach uses the reference without modification and controls how much guidance the model receives through how much of the reference is revealed in the prefix.

However, when generating a trajectory following a prefix, deciding how much to reveal involves a trade-off. Specifically, using a longer reference prefix, which provides more guidance, makes the continuation more likely to be correct. Yet the resulting trajectory contains more content the model would not generate itself, making learning harder. Using a shorter prefix keeps the trajectory closer to the model's own generation, but increases the chance of failure. At the two extremes, revealing the entire reference reduces to training on the reference directly, while revealing nothing reduces to standard RLVR, which provides no learning signal on questions the model cannot yet solve. Balancing the two is therefore essential. The prefix should reveal just enough for the model to succeed, and leave the rest for the model to generate on its own.

To analyze this trade-off, we first ask which trajectories are ideal to learn from. Ideally, these trajectories should be correct and generated entirely by the model. We therefore define the distribution of the model's own correct samples as the \emph{ideal distribution}. We show that the Kullback-Leibler (KL) divergence from the distribution of generated trajectories used for training to this ideal distribution decreases as $d_tS_t$ increases, up to a small bounded term. Here $d_t$ is the probability that the model follows the reference's prefix to generate a successful trajectory, and $S_t$ is the negative log probability of the remaining reference given the prefix. Intuitively, $d_t$ reflects how much the prefix helps the model succeed, while $S_t$ indicates how unfamiliar the reference suffix is. Moreover, $d_tS_t$ captures both sides of the trade-off. A longer prefix tends to increase $d_t$, since more guidance makes success more likely, but it decreases $S_t$, since less of the reference is left.

As a result, a good prefix should maximize $d_tS_t$. However, $d_t$ is unknown, and estimating it for every prefix of every question is impractical. We avoid this estimation by learning a \emph{prefix selector}, a distribution over a few guidance levels shared across questions. Since the selector only needs to work well on average, the outcome of each continuation already provides the feedback it needs, without any additional generation.

These analyses lead to \emph{Adaptive Reference Guidance} (\textbf{\method}). Given a question and its reference, \method samples guidance levels from the prefix selector, generates continuations from the resulting prefixes within a fixed budget, and falls back to the reference when no continuation succeeds. The same continuations also update the prefix selector, which therefore requires no additional generation. To examine how the accepted trajectories affect policy learning, we apply \method to Group Relative Policy Optimization (GRPO) \citep{shao2024deepseekmath} for groups in which every sampled response fails, resulting in no learning signals.

Empirically, on Qwen3-4B and Qwen3-8B \citep{yang2025qwen3} across five mathematical reasoning benchmarks, \method achieves the highest aggregate $\passk$ of $67.33\%$ and $70.00\%$, respectively. Moreover, its average sampled accuracy remains competitive, with a small gap to the strongest baseline. Our contributions are threefold.
\begin{itemize}
\item We analyze reference-prefix continuation through the KL divergence to the model's own correct trajectories, revealing a trade-off between reference guidance and surprisal.
\item We derive an unbiased update of the prefix selector from the same continuations, adapting the guidance level across questions and during training without additional generation.
\item We develop \method from this analysis and apply it to all-failure groups in GRPO, achieving the highest aggregate $\passk$ across five math benchmarks on two model scales.
\end{itemize}

\section{Preliminaries}
\label{sec-preliminaries}

\noindent\textbf{Reasoning with verifiable rewards.}
Let $\pi_\theta(y\mid x)$ be a language model that generates a trajectory $y$ for a question $x$ drawn from the training distribution $\mathcal D$. The trajectory includes the reasoning, the final answer, and the termination token. A verifier assigns a binary reward $R_{\mathrm{acc}}(x,y)\in\{0,1\}$, and we call $y$ \emph{correct} when $R_{\mathrm{acc}}(x,y)=1$. Let $C_x$ be the set of correct trajectories and $p_\theta(x)=\pi_\theta(C_x\mid x)$ be the success probability. During training, $\theta_{\mathrm{old}}$ denotes the parameters of the current model, which generates the samples at each training step. We use natural logarithms throughout.

\noindent\textbf{Reference-prefix continuation with fallback.}
Suppose that the question comes with a correct reference $r\in C_x$. For a prefix length $0\le t<|r|$, we write $r=h_t\oplus a_t$, where the reference prefix $h_t$ contains the first $t$ tokens of $r$, the reference suffix $a_t$ contains the remaining tokens, and $\oplus$ denotes token concatenation. The model samples a continuation $Z\sim\pi_\theta(\cdot\mid x,h_t)$, which forms the \emph{candidate trajectory} $\widetilde Y_t=h_t\oplus Z$. At $t=0$, the prefix is empty and the model generates under the original question alone. We keep the candidate trajectory if it is correct and fall back to the reference otherwise, which gives the accepted trajectory
\begin{equation}
Y=\begin{cases}
\widetilde Y_t,&\widetilde Y_t\in C_x,\\
r,&\widetilde Y_t\notin C_x,
\end{cases}
\qquad Y\sim Q_t.
\label{eq-acquisition}
\end{equation}
Thus, $Q_t$ assigns all probability to correct trajectories for every prefix length. In contrast, training on the reference corresponds to the distribution $\delta_r$, which assigns all probability to $r$.

\noindent\textbf{Ideal distribution.}
As discussed in Section~\ref{sec-introduction}, we take the model's own correct trajectories as the ideal distribution. Formally, for $p_\theta(x)>0$, sampling from $\pi_\theta(\cdot\mid x)$ and keeping only correct trajectories gives
\begin{equation}
\nu_\theta(y\mid x)=\frac{\pi_\theta(y\mid x)\mathbf1\{y\in C_x\}}{p_\theta(x)}.
\label{eq-target}
\end{equation}
This distribution also has a direct connection to likelihood learning \citep{phan2023training,singh2023beyond}. Specifically, for a finite trajectory space and a differentiable policy with positive probabilities,
\begin{equation}
\mathbb E_{Y\sim\nu_\theta(\cdot\mid x)}
\left[\nabla_\theta\log\pi_\theta(Y\mid x)\right]
=\nabla_\theta\log p_\theta(x).
\label{eq-success-score}
\end{equation}
In other words, learning the likelihood of trajectories drawn from $\nu_\theta$ increases the success probability in expectation at the current model. Finally, we measure how far a distribution $Q$ of correct trajectories is from the ideal distribution by the Kullback-Leibler (KL) divergence
\[
D_{\mathrm{KL}}(Q\Vert\nu_\theta)
=\mathbb E_{Y\sim Q}\left[\log\frac{Q(Y)}{\nu_\theta(Y\mid x)}\right].
\]
Appendix~\ref{app-likelihood} derives Equation~\eqref{eq-success-score} and explains this direction of KL.

\section{Analyzing and Selecting Reference Prefixes}
\label{sec-theory}

In this section, we analyze how the prefix length affects the accepted trajectory and turn the analysis into a practical way of choosing prefixes. Section~\ref{sec-target} derives the distance from the accepted trajectory to the ideal distribution, which identifies the preferred prefix of a question. Section~\ref{sec-challenges} explains why this prefix is difficult to find during training, and Section~\ref{sec-shared-selection} addresses both challenges by learning a prefix selector that performs well on average across questions.

\subsection{Distance to the Ideal Distribution}
\label{sec-one-attempt}
\label{sec-target}

Fix a question $x$ with reference $r$ and the current model $\pi_{\theta_{\mathrm{old}}}$, and suppress the conditioning on $x$ and the parameters by writing $\pi$, $C$, $p=\pi(C)$, and $\nu$. We assume a finite trajectory space with $0<\pi(r)<1$. For a prefix length $t$, we define the \emph{successful replacement probability} $d_t$ and the \emph{reference surprisal} $S_t$ as
\begin{equation}
d_t=\Pr\bigl(\widetilde Y_t\in C,\ \widetilde Y_t\ne r\bigr),
\qquad S_t=-\log\pi(a_t\mid h_t).
\label{eq-local-quantities}
\end{equation}
Here $d_t$ excludes an exact copy of the reference, since the accepted trajectory is then still $r$. In other words, $d_t$ is the probability that the model replaces the reference suffix with a different correct continuation, and $S_t$ measures how unfamiliar the reference suffix is to the model after the prefix.

\begin{theorem}[Distance to the ideal distribution]
\label{thm-one-attempt}
Under the setup above, for every prefix length $t$,
\begin{equation}
D_{\mathrm{KL}}(Q_t\Vert\nu)
=\underbrace{\log\frac{p}{\pi(r)}}_{D_{\mathrm{KL}}(\delta_r\Vert\nu)}
-d_tS_t+(1-d_t)\log(1-d_t).
\label{eq-kl-closed-form}
\end{equation}
Moreover, every $t^*\in\arg\max_t d_tS_t$ satisfies
$D_{\mathrm{KL}}(Q_{t^*}\Vert\nu)\le\min_t D_{\mathrm{KL}}(Q_t\Vert\nu)+1/e$.
\end{theorem}

Theorem~\ref{thm-one-attempt} expresses the distance between $Q_t$ and the ideal distribution $\nu$ in closed form with three terms. The first term equals that of training on the reference, which does not depend on the prefix. The third term is bounded within $[-1/e, 0]$. Therefore, the prefix affects the distance mainly through the second term, $d_tS_t$. Intuitively, this product formalizes the trade-off in choosing the prefix length. A longer prefix never increases $S_t$, since it leaves less of the reference for the continuation to replace. In contrast, a longer prefix can increase $d_t$ by making the continuation easier to succeed. Consequently, both a nearly complete prefix and an empty prefix on a hard question make $d_tS_t$ close to zero and leave the distance close to $D_{\mathrm{KL}}(\delta_r\Vert\nu)$, indicating learning from the reference alone. Moreover, a prefix with a higher success probability of the continuation need not give a smaller distance, as illustrated in Appendix~\ref{app-example}. Instead, maximizing $d_tS_t$ attains the smallest distance up to $1/e$ by the second statement in the theorem, and we call $h_{t^*}$ the \emph{preferred prefix} of the question under the current model. Appendix~\ref{app-proof} gives the proof.

\subsection{Challenges in Finding the Preferred Prefix}
\label{sec-challenges}

Although Theorem~\ref{thm-one-attempt} identifies the preferred prefix, finding it during training is difficult. The reference surprisal $S_t$ is available for every prefix length from a single forward pass of the model on the reference. However, the successful replacement probability $d_t$ is unknown, which leads to the following two challenges.

\noindent\textbf{Challenge 1: Variation across questions and training steps.}
Theorem~\ref{thm-one-attempt} concerns a fixed question and a fixed model, and $d_t$ depends on both. Therefore, the estimates of $d_t$ for one question do not apply to another question. Furthermore, since $d_t$ changes as the model is updated, the estimates obtained at one training step may become outdated at the next. As a result, the search for $t^*$ would have to be repeated for every question and every model update.

\noindent\textbf{Challenge 2: Limited sampling budget.}
Moreover, each such search is expensive. Identifying $t^*$ compares $d_tS_t$ across prefix lengths and therefore requires an accurate estimate of $d_t$ at every prefix length. Since $d_t$ is often small on hard questions, such an estimate requires many continuations even for a single prefix length, and this cost multiplies across prefix lengths.

\subsection{Learning a Prefix Selector across Questions}
\label{sec-shared-selection}
\label{sec-observation}

Both challenges arise from seeking the preferred prefix of each question at each training step. Instead, we learn a way of choosing prefixes that performs well on average over the questions that receive guidance, which addresses both challenges.

\noindent\textbf{Sharing across questions and training steps.}
Averaging over questions requires describing prefixes in the same way for every question, which the prefix length does not achieve for references of different lengths. Since Theorem~\ref{thm-one-attempt} measures the replaced suffix by its surprisal rather than its length, we describe a prefix through the reference surprisal. Specifically, for a question $x$ with reference $r_x$, let $S_{0,x}=-\log\pi_{\theta_{\mathrm{old}}}(r_x\mid x)$ be the surprisal of the entire reference. A \emph{guidance level} $\alpha$ in a set $\mathcal A\subset(0,1]$ selects a prefix length $t_x(\alpha)$ whose reference surprisal is approximately $\alpha S_{0,x}$. Thus, $\alpha$ is the fraction of the reference surprisal left to the continuation, where $\alpha=1$ gives the empty prefix. As with prefix lengths, the levels cannot all be tried within the sampling budget. Therefore, $\mathcal A$ contains only a few levels, and Section~\ref{sec-method-guidance} specifies the mapping from levels to prefix lengths.

We then learn a \emph{prefix selector} $\rho_\phi$, which is a distribution over $\mathcal A$ with parameters $\phi$ shared across questions. Let $\mathcal D_{\mathrm{guide}}$ denote the distribution of questions that receive guidance at the current training step, which Section~\ref{sec-method-learning} specifies for our application. For $X\sim\mathcal D_{\mathrm{guide}}$ and $\alpha\sim\rho_\phi$, we write $d_{X,\alpha}$, $S_{X,\alpha}$, and $Q_{X,\alpha}$ for $d_t$, $S_t$, and $Q_t$ at the prefix length $t_X(\alpha)$, and $\nu_X$ for the ideal distribution of $X$. The selector maximizes
\begin{equation}
J(\phi)=\mathbb E_{X\sim\mathcal D_{\mathrm{guide}},\,\alpha\sim\rho_\phi}
\left[\frac{d_{X,\alpha}S_{X,\alpha}}{S_{0,X}}\right],
\label{eq-shared-objective}
\end{equation}
which averages $d_tS_t$ over the questions and the sampled levels in units of the surprisal of the entire reference. This normalization keeps questions with long or unfamiliar references from dominating the average, and it does not change which prefix maximizes $d_tS_t$ within any question. Since the selector is shared across questions and updated at every training step, continuations on all questions and at all steps inform the same choice, while the selected prefix still adapts to each question through its reference surprisal. This addresses Challenge 1, and Appendix~\ref{app-selector-sharing} bounds the sharing cost.

\noindent\textbf{Learning without estimating $d_t$.}
The average objective also avoids the accurate estimates of $d_t$ required in Challenge 2. For a continuation from the prefix $h_t$, we define the \emph{selector feedback}
\begin{equation}
W_t=\mathbf1\bigl\{\widetilde Y_t\in C,\ \widetilde Y_t\ne r\bigr\}\frac{S_t}{S_0},
\label{eq-observation}
\end{equation}
and write $W_{X,\alpha}$ for the feedback at the prefix length $t_X(\alpha)$. In other words, a successful replacement receives the fraction of the reference surprisal left to the continuation, while a failure or an exact copy of the reference receives zero. Since the indicator has expectation $d_t$, the feedback satisfies
\begin{equation}
\mathbb E[W_t]=\frac{d_tS_t}{S_0}.
\label{eq-observation-bound}
\end{equation}
Thus, the feedback of a single continuation is an unbiased estimate of its term in $J$. A single estimate is noisy, and noisy estimates cannot reliably rank the prefix lengths of one question. However, since $J$ is an average that is linear in $d_{X,\alpha}$, the errors of single continuations cancel when pooled over many questions. Moreover, since $\mathcal A$ contains only a few levels and the selector keeps a positive probability on every level, every level receives feedback from many questions, which addresses Challenge 2. The following theorem shows that $J$ still measures the distance in Theorem~\ref{thm-one-attempt} and that the selector can be learned from this feedback.

\begin{theorem}[Learning the prefix selector from feedback]
\label{thm-selector}
\label{prop-shared-selection}
Fix the current model, the references, and $\mathcal D_{\mathrm{guide}}$, where every question satisfies the setup of Section~\ref{sec-target}. Let $\mathcal L(\phi)$ and $\mathcal L_{\mathrm{ref}}$ be the expectations of $D_{\mathrm{KL}}(Q_{X,\alpha}\Vert\nu_X)/S_{0,X}$ and $D_{\mathrm{KL}}(\delta_{r_X}\Vert\nu_X)/S_{0,X}$ over $X\sim\mathcal D_{\mathrm{guide}}$ and $\alpha\sim\rho_\phi$. Then
\begin{equation}
J(\phi)\le\mathcal L_{\mathrm{ref}}-\mathcal L(\phi)\le J(\phi)+\mathbb E_X\left[\frac{1}{eS_{0,X}}\right].
\label{eq-selector-distance}
\end{equation}
Moreover, for a differentiable prefix selector with positive probabilities and a baseline $b$ fixed before sampling,
\begin{equation}
\mathbb E\left[(W_{X,\alpha}-b)\nabla_\phi\log\rho_\phi(\alpha)\right]
=\nabla_\phi J(\phi).
\label{eq-shared-gradient}
\end{equation}
\end{theorem}

The first statement shows that, compared with training on the reference, the selector reduces the average normalized distance by at least $J$, where each distance is divided by the surprisal of the entire reference. The gap is at most $\mathbb E_X[1/(eS_{0,X})]$, which is small when no reference is highly likely under the model. Therefore, maximizing $J$ pursues the goal of Theorem~\ref{thm-one-attempt} across questions. The second statement shows that raising the probability of the levels whose feedback exceeds the baseline $b$ moves the selector toward a larger $J$ on average, where $b$ introduces no bias and serves to reduce the variance. Since these continuations also produce the accepted trajectories, the update requires no additional generation. Appendix~\ref{app-selector} proves Theorem~\ref{thm-selector} and extends it to batches, changing models, and randomly sampled levels.

\section{ARG: Adaptive Reference Guidance}
\label{sec-method}

We now build \method on Theorems~\ref{thm-one-attempt} and~\ref{thm-selector}. Section~\ref{sec-method-guidance} describes how \method generates accepted trajectories for a question, and Section~\ref{sec-method-update} describes how the same continuations update the prefix selector. Neither step depends on a particular reinforcement learning algorithm. The algorithm only determines the questions that receive guidance and how the model learns from the accepted trajectories, and Section~\ref{sec-method-learning} specifies both for GRPO.

\subsection{Generating Accepted Trajectories}
\label{sec-method-guidance}

Consider a question $x$ that receives guidance, with reference $r$. \method first scores $r$ once under the current model $\pi_{\theta_{\mathrm{old}}}$, which gives the reference surprisal $S_t$ for every prefix length. Each guidance level $\alpha\in\mathcal A$ then maps to a prefix length $t_x(\alpha)$. Specifically, $\alpha=1$ gives the empty prefix. For $\alpha<1$, we take the latest prefix length whose reference surprisal is at least $\alpha S_0$ and move it to a nearby sentence or formula boundary, so that the prefix ends at a natural break. Appendix~\ref{app-boundaries} gives the boundary rules.

The prefix selector assigns a logit $\phi_\alpha$ to each level and samples from
\begin{equation}
\rho_\phi(\alpha)
=(1-\varepsilon_{\mathrm{sel}})
\frac{\exp(\phi_\alpha)}
{\sum_{\alpha'\in\mathcal A}\exp(\phi_{\alpha'})}
+\frac{\varepsilon_{\mathrm{sel}}}{|\mathcal A|},
\label{eq-selector}
\end{equation}
where $0<\varepsilon_{\mathrm{sel}}<1$ keeps every level observed as the model changes. Within a fixed budget of $K$ continuations per question, we independently sample $\alpha_j\sim\rho_\phi$ for $j=1,\ldots,K$. For the resulting prefix length $t_j=t_x(\alpha_j)$, we sample a continuation $Z_j\sim\pi_{\theta_{\mathrm{old}}}(\cdot\mid x,h_{t_j})$ and verify the candidate trajectory $\widetilde Y_j=h_{t_j}\oplus Z_j$.

Every correct candidate trajectory different from $r$ can serve as an accepted trajectory. If there is none, \method falls back to the reference. Therefore, the question always receives a correct trajectory. When the number of accepted trajectories is limited, as in our application, we prefer candidate trajectories that keep less of the reference. Specifically, a correct candidate trajectory from the empty prefix keeps none of the reference. Indeed, it is generated entirely by the model and follows the ideal distribution restricted to trajectories other than $r$. Therefore, we accept these candidate trajectories first. Among candidate trajectories from nonempty prefixes, we prefer the shortest prefix, since the surprisal of the prefix, $-\log\pi_{\theta_{\mathrm{old}}}(h_t\mid x)=S_0-S_t$, never decreases with $t$. Finally, Theorem~\ref{thm-one-attempt} concerns a single continuation. For accepted trajectories from several continuations, Appendix~\ref{app-candidate-selection} gives a sufficient condition under which the distance remains no larger than that of training on the reference.

\subsection{Updating the Prefix Selector}
\label{sec-method-update}

Every continuation provides the selector feedback $W_j=W_{t_j}$ in Equation~\eqref{eq-observation}, whether or not its candidate trajectory is accepted. Since the feedback does not depend on which trajectories are accepted, the second statement of Theorem~\ref{thm-selector} holds under any acceptance rule, while the first statement, like Theorem~\ref{thm-one-attempt}, concerns a single continuation.

We collect continuations for all questions in a batch while keeping the current model and the prefix selector fixed. Let $\mathcal U$ index these continuations, with sampled levels $\alpha_j$ and selector feedback $W_j$. For a nonempty $\mathcal U$, we update the prefix selector by
\begin{equation}
\phi\leftarrow\phi+\eta_{\mathrm{sel}}
\sum_{j\in\mathcal U}(W_j-b)
\nabla_\phi\log\rho_\phi(\alpha_j),
\label{eq-selector-update}
\end{equation}
where $\eta_{\mathrm{sel}}>0$ is the step size and $b$ is a running average of past feedback fixed before collection. By the second statement of Theorem~\ref{thm-selector}, the update follows the gradient of $J$ in expectation. Appendix~\ref{app-selector} gives the baseline update and derives the expected update as the model and the questions change during training.

\subsection{Applying \method to GRPO}
\label{sec-method-learning}

GRPO \citep{shao2024deepseekmath} samples a group of $G=8$ responses for each question and compares their rewards within the group. When every response in a group is incorrect, the group provides no signal for improving correctness, since all correctness rewards are identical. In such an \emph{all-failure group}, adding a correct trajectory restores the learning signal. Therefore, we apply \method to all-failure groups, using $K=8$ continuations. These questions define $\mathcal D_{\mathrm{guide}}$.

We insert at most four accepted trajectories into each group, which keeps at least four original failures and permits a balanced group, motivated by prior work on intermediate pass rates \citep{bae2026online,liu2026restoring}. Following the preference in Section~\ref{sec-method-guidance}, we first accept correct candidate trajectories from the empty prefix. If fewer than four are accepted, we additionally accept at most one from a nonempty prefix, which limits repeated reference content in the same group. If no candidate trajectory is accepted, we insert the reference once as fallback. Each inserted trajectory replaces an original failure at a uniformly chosen position. We then recompute the group advantages $A_i$ from the final correctness and format rewards, as specified in Appendix~\ref{app-grpo-loss}.

The model learns each inserted trajectory under the original question through its complete likelihood, including the prefix. This matches Theorem~\ref{thm-one-attempt}, which measures the complete accepted trajectory under the original question. Let $\mathcal I$ index the inserted trajectories $(x_i,y_i)$ and let $T$ be the total number of tokens in the loss. Their contribution to the gradient is
\begin{equation}
\nabla_\theta\mathcal L_{\mathrm{inserted}}
=-\frac1T\sum_{i\in\mathcal I}A_i
\nabla_\theta\log\pi_\theta(y_i\mid x_i),
\label{eq-acquired-nll}
\end{equation}
where the trajectories, advantages, and token count are fixed during differentiation. The remaining original responses use the clipped GRPO loss with the same denominator. Moreover, for binary rewards, Appendix~\ref{app-positive-learning} shows that inserting $Y\sim Q$ into an all-failure group gives an expected update proportional to the gradient of $\log\frac{p_\theta}{1-p_\theta}-D_{\mathrm{KL}}(Q\Vert\nu_\theta)$ at the current model. The gradient of the second term vanishes when $Q$ is the ideal distribution, in which case the update only increases the odds of success. This connects the distance in Theorem~\ref{thm-one-attempt} to the GRPO update.

After collecting continuations, we update the prefix selector by Equation~\eqref{eq-selector-update} and take one policy update on the final groups. Algorithm~\ref{alg-guidance} summarizes the training procedure.

\section{Experiments}
\label{sec-experiments}

\subsection{Experimental Setting}
\label{sec-experimental-setting}

\noindent\textbf{Data and models.}
We use DeepMath-103k with offline reference solutions \citep{he2026deepmath} as our training set. We use Qwen3-4B and Qwen3-8B with thinking mode disabled \citep{yang2025qwen3} as our base LLMs. For all methods, we evaluate on five mathematical reasoning benchmarks, AIME 2024, AIME 2025, AIME 2026, HMMT February 2025, and HMMT November 2025 \citep{balunovic2026matharena,dekoninck2026beyond}. Appendix~\ref{app-implementation} gives further details.

\noindent\textbf{Baselines and evaluation.}
We compare against the initial Qwen model, vanilla supervised fine-tuning (SFT), GRPO \citep{shao2024deepseekmath}, OPSD \citep{zhao2026self}, SAGE \citep{liao2026self}, BREAD \citep{zhang2026bread}, and LUFFY \citep{yan2026learning}. During evaluation, we draw 12 responses per question. $\passk$ measures whether a question has at least one correct response, while $\avgk$ measures the fraction of correct responses. We use $\passk$ as the primary metric.

\begin{table}[!htpb]
\centering
\small
\setlength{\tabcolsep}{4pt}
\begin{tabular*}{\linewidth}{@{\hspace{6pt}\extracolsep{\fill}}ll*{5}{c}|c@{\hspace{6pt}}}
\toprule
\textbf{Model} & \textbf{Method} & \textbf{AIME24} & \textbf{AIME25} & \textbf{AIME26} & \textbf{HMMT\_Feb} & \textbf{HMMT\_Nov} & \textbf{Avg.} \\
\midrule
\multirow{8}{*}{Qwen3-4B}
 & Base & 50.00 & 43.33 & 43.33 & 23.33 & 33.33 & 38.67 \\
 & SFT & 63.33 & 63.33 & 66.67 & 36.67 & 36.67 & 53.33 \\
 & GRPO & \underline{76.67} & \underline{70.00} & \textbf{73.33} & \textbf{53.33} & 46.67 & 64.00 \\
 & OPSD & 56.67 & 36.67 & 40.00 & 40.00 & 40.00 & 42.67 \\
 & SAGE & 73.33 & \underline{70.00} & \underline{70.00} & 43.33 & \underline{56.67} & 62.67 \\
 & BREAD & \textbf{80.00} & \textbf{73.33} & \textbf{73.33} & \underline{50.00} & \underline{56.67} & \underline{66.67} \\
 & LUFFY & \textbf{80.00} & \textbf{73.33} & \underline{70.00} & 43.33 & \underline{56.67} & 64.67 \\
 & \textbf{\method (ours)} & \textbf{80.00} & \textbf{73.33} & \textbf{73.33} & 43.33 & \textbf{66.67} & \textbf{67.33} \\
\midrule
\multirow{8}{*}{Qwen3-8B}
 & Base & 46.67 & 43.33 & 53.33 & 16.67 & 46.67 & 41.33 \\
 & SFT & \underline{80.00} & 70.00 & \textbf{73.33} & 46.67 & 56.67 & 65.33 \\
 & GRPO & \underline{80.00} & 63.33 & \underline{70.00} & 46.67 & 53.33 & 62.67 \\
 & OPSD & 66.67 & 60.00 & 50.00 & 36.67 & 43.33 & 51.33 \\
 & SAGE & \underline{80.00} & 70.00 & \underline{70.00} & 50.00 & 60.00 & 66.00 \\
 & BREAD & \underline{80.00} & 66.67 & 66.67 & \underline{53.33} & \textbf{70.00} & \underline{67.33} \\
 & LUFFY & 76.67 & \textbf{76.67} & \textbf{73.33} & 46.67 & \underline{63.33} & \underline{67.33} \\
 & \textbf{\method (ours)} & \textbf{83.33} & \underline{73.33} & \textbf{73.33} & \textbf{56.67} & \underline{63.33} & \textbf{70.00} \\
\bottomrule
\end{tabular*}
\caption{Main experimental results in $\passk$ percentages. Avg. is the mean over the five benchmarks. Within each model, \textbf{bold} and \underline{underline} mark the best and second-best values in each column.}
\label{tab-main-results}
\vspace{-10pt}
\end{table}

\subsection{Results}
\label{sec-main-performance}

Table~\ref{tab-main-results} evaluates \method through its application to GRPO. \method achieves the highest aggregate $\passk$ at both model scales, reaching $67.33\%$ on Qwen3-4B and $70.00\%$ on Qwen3-8B. Compared with GRPO, these results improve $\passk$ by $3.33$ and $7.33$ percentage points, respectively. \method also exceeds the reference-informed baselines in aggregate $\passk$, which use reference solutions as off-policy responses in LUFFY, as prefixes in BREAD, and as hints in SAGE \citep{yan2026learning,zhang2026bread,liao2026self}. Thus, \method solves the largest fraction of evaluation questions under the 12-response budget. For $\avgk$, \method improves over GRPO on both models and remains competitive with the strongest baseline, as discussed in Appendix~\ref{app-main-results}. Sections~\ref{sec-ablations} and~\ref{sec-training-dynamics} examine the contributions of the continuations and the prefix selector.

\subsection{Ablations}
\label{sec-ablations}

We examine the contribution of the continuations, compare with unguided resampling, and evaluate the surprisal weighting in the selector feedback.

\noindent\textbf{Reference insertion.}
For each all-failure group, we replace a random failure in the original group with the reference without any extra generation. This ablation corresponds to training on the reference in Theorem~\ref{thm-one-attempt}. Table~\ref{tab-ablation-construction} shows that reference insertion achieves $56.67\%$ in $\passk$ and $28.44\%$ in $\avgk$ on Qwen3-4B, compared with $67.33\%$ and $41.11\%$ for \method. These results support using the reference to guide continuations and reserving the reference itself for fallback.

\noindent\textbf{Resampling.}
We allocate the same eight continuations to the empty prefix for each all-failure group. Correct candidate trajectories replace original failures, and groups with no correct candidate trajectory remain unchanged. Thus, this comparison evaluates \method against additional sampling without reference guidance or fallback. Table~\ref{tab-ablation-construction} shows that \method improves $\passk$ from $60.00\%$ to $67.33\%$ and $\avgk$ from $34.00\%$ to $41.11\%$ on Qwen3-4B. During their respective training runs, candidate trajectories are accepted in $66.0\%$ of all-failure groups with \method and only $32.5\%$ with resampling. More details are shown in Appendix~\ref{app-ablation-generation}.

\noindent\textbf{Without surprisal weighting.}
We replace the selector feedback in Equation~\eqref{eq-observation} with the successful replacement indicator. The guidance levels, surprisal-based prefix construction, selector update rule, and acceptance rule remain unchanged. Table~\ref{tab-ablation-weighting} shows that removing the weighting lowers both evaluation metrics on both models, with $\passk$ decreasing by $4.00$ and $6.00$ percentage points on Qwen3-4B and Qwen3-8B, respectively. These results support the surprisal weighting in the selector feedback. Section~\ref{sec-training-dynamics} examines the selector feedback and prefix selection under \method, while Appendix~\ref{app-ablation-weighting} compares the prefix selection with that of the ablation.

\begin{table}[t]
\centering
\small
\setlength{\tabcolsep}{1.8pt}
\renewcommand\thesubtable{(\alph{subtable})}
\captionsetup[subtable]{labelformat=simple,labelsep=space}
\begin{subtable}[t]{0.49\linewidth}
\centering
\begin{tabular}{@{}lcc@{}}
\toprule
Method & $\passk$ & $\avgk$ \\
\midrule
Reference insertion & 56.67 & 28.44 \\
Resampling & 60.00 & 34.00 \\
\method & 67.33 & 41.11 \\
\bottomrule
\end{tabular}
\caption{Sources of accepted trajectories on Qwen3-4B}
\label{tab-ablation-construction}
\end{subtable}\hfill
\begin{subtable}[t]{0.49\linewidth}
\centering
\begin{tabular}{@{}lcccc@{}}
\toprule
\multirow{2}{*}{Method} & \multicolumn{2}{c}{Qwen3-4B} & \multicolumn{2}{c}{Qwen3-8B} \\
\cmidrule(lr){2-3}\cmidrule(lr){4-5}
& $\passk$ & $\avgk$ & $\passk$ & $\avgk$ \\
\midrule
w/o weighting & 63.33 & 34.67 & 64.00 & 31.89 \\
\method & 67.33 & 41.11 & 70.00 & 42.28 \\
\bottomrule
\end{tabular}
\caption{Surprisal weighting on both models}
\label{tab-ablation-weighting}
\end{subtable}
\caption{Ablation results averaged over the five benchmarks. Resampling uses eight continuations from the empty prefix without reference fallback. The weighting ablation removes $S_t/S_0$ from the selector feedback.}
\label{tab-ablation}
\vspace{-10pt}
\end{table}

\subsection{Training Dynamics}
\label{sec-training-dynamics}
Figure~\ref{fig-guidance-dynamics} examines Qwen3-8B training dynamics through successful replacement rates, selector feedback, and prefix selection.

\noindent\textbf{Successful replacement and selector feedback.}
Figure~\ref{fig-dynamics-replacement-8b} reports the successful replacement rate for each guidance level. Each point averages all continuations at that level within a window of 50 training steps. The level $\alpha=1/4$ has a higher successful replacement rate than $\alpha=1/2$ in every window. However, Figure~\ref{fig-dynamics-feedback-8b} shows that their ordering can change when the successful replacement indicator is weighted by the fraction $S_t/S_0$ of the reference surprisal left to the continuation. In the final window, for example, $\alpha=1/4$ has lower mean feedback than $\alpha=1/2$ ($0.0802$ versus $0.0991$), although its successful replacement rate is higher ($43.84\%$ versus $23.65\%$). Thus, the feedback favors less reference guidance in this window, consistent with Section~\ref{sec-target}, where a higher success probability of the continuation need not give a smaller distance.

\begin{figure}[t]
    \centering
    \captionsetup[subfigure]{labelformat=simple,labelsep=space}
    \includegraphics[width=\linewidth]{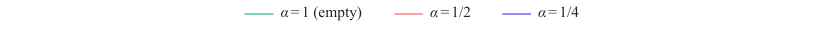}\par
    \begin{subfigure}{0.32\linewidth}
        \includegraphics[width=\linewidth]{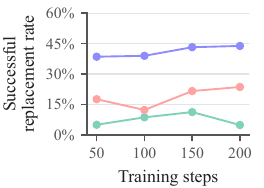}
        \caption{Successful replacement}
        \label{fig-dynamics-replacement-8b}
    \end{subfigure}\hfill
    \begin{subfigure}{0.32\linewidth}
        \includegraphics[width=\linewidth]{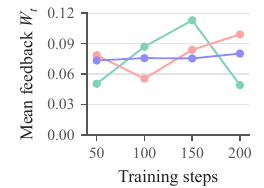}
        \caption{Selector feedback}
        \label{fig-dynamics-feedback-8b}
    \end{subfigure}\hfill
    \begin{subfigure}{0.32\linewidth}
        \includegraphics[width=\linewidth]{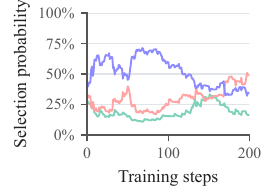}
        \caption{Prefix selection}
        \label{fig-dynamics-selector-8b}
    \end{subfigure}
    \caption{Training dynamics of a Qwen3-8B \method run. Panel (a) reports successful replacement rates, panel (b) reports mean selector feedback $W_t$, and panel (c) shows the selection probabilities. In panels (a) and (b), a point at step $s$ summarizes all continuations at that level across the past 50 steps, including failures and candidate trajectories not accepted for training. Panel (c) shows the state after each training step without smoothing.}
    \label{fig-guidance-dynamics}
    \vspace{-10pt}
\end{figure}

\noindent\textbf{Prefix selection during training.}
The selector uses the feedback $W_t$ from each continuation to update its selection probabilities. Figure~\ref{fig-dynamics-selector-8b} shows how these probabilities change during training. The selector favors $\alpha=1/4$ through much of training and increases its probability of selecting $\alpha=1/2$ later, assigning it the highest final probability of $49.0\%$. Since a larger $\alpha$ leaves more of the reference surprisal to the continuation, this shift supplies less reference guidance. Appendix~\ref{app-training-dynamics} provides additional training-dynamics results.

\section{Related Work}
\label{sec-related}

\noindent\textbf{Learning signals beyond on-policy rollouts.}
LUFFY, SRFT, ReLIFT, and related methods combine reference or teacher trajectories with the model's own rollouts, through mixed-policy groups or combined supervised and reinforcement updates \citep{yan2026learning,fu2026srft,ma2026learning,zhang2026policy,liu2025superrl,dong2026rl,lv2025towards}. OPSD and SDPO instead use the model conditioned on solution information or feedback as a teacher for token-level supervision \citep{zhao2026self,hubotter2026reinforcement}. Other methods recover learning signals from failed rollouts, either by repairing them into correct ones \citep{li2026learning,yang2026int,zhang2025critique,ren2026recycling} or by reshaping the rewards or advantages of groups with identical rewards \citep{chen2025stepwise,nie2026gradient,nan2025ngrpo,le2026no}.

\noindent\textbf{Reference-guided generation.}
Closer to our work, many methods use the reference only to guide the model's own generation, either by revealing a prefix of the reference for the model to continue \citep{liu2026uft,setlur2026reuse,qu2026pope,zhang2026stephint,su2026trust,liu2025ghpo} or by providing hints derived from the reference or its answer \citep{liao2026self,zhang2026scaf,nath2025adaptive,chen2026nudging,zhou2026expo}. The amount of guidance is fixed or scheduled \citep{xi2024training,huang2025blending,li2026questa}, adapted to the observed success of the model \citep{zhang2026bread, ICLR2026_c073a9d0,li2025staying,guo2026g2rpo,beliaev2026max}, or chosen by the gain of a prefix or the transfer of a hint \citep{zhou2026correctness,xia2026learning}. In contrast, we derive the amount of guidance from the distance between the accepted trajectory and the ideal distribution, which shows that success alone does not determine the preferred prefix, and we learn a prefix selector shared across questions from the continuations themselves.

\noindent\textbf{Posterior views of learning from correct trajectories.}
Training on the model's own correct samples, as in STaR and ReST$^{\mathrm{EM}}$, approximates learning from our ideal distribution \citep{zelikman2022star,singh2023beyond}, which is also the correct distribution closest to the model in KL and relates to the Bayesian view of KL-regularized RL \citep{dymetman2026binary,korbak2022rl}. Latent-variable methods approximate the posterior over rationales given the correct answer \citep{phan2023training,zhou2026variational}, and TRICE keeps a stored rationale until a new proposal is correct, similar to our fallback. Other work analyzes or modifies the function of the success probability that RL maximizes \citep{davis2025objective,tajwar2026maximum,ramasubramanian2026tail}, or shows that targets unlikely under the model are harder to learn from \citep{ren2024learn,li2025small,setlur2024rl}.

\section{Conclusion and Limitations}
\label{sec-conclusion}

We study reference-prefix continuation with fallback through the distance from the accepted trajectory to the ideal distribution. Up to a small bounded term, this distance decreases with $d_tS_t$, which captures the trade-off between successful replacement and the reference surprisal. This analysis leads to \method, which learns a prefix selector across questions from the continuations themselves. Applied to all-failure groups in GRPO, \method achieves the highest aggregate $\passk$ on Qwen3-4B and Qwen3-8B across five mathematical benchmarks, with competitive average sampled accuracy.

\noindent\textbf{Limitations.}
\method is only applied to GRPO, although the method does not depend on a particular reinforcement learning algorithm. Moreover, the selector chooses among a predefined set of guidance levels, and our experiments cover mathematical reasoning with one model family. Future work may extend the method to other domains with verifiable rewards.

\bibliographystyle{iclr2027_conference}
\bibliography{iclr2027_conference}

\begin{thebibliography}{61}
\providecommand{\natexlab}[1]{#1}
\providecommand{\url}[1]{\texttt{#1}}
\expandafter\ifx\csname urlstyle\endcsname\relax
  \providecommand{\doi}[1]{doi: #1}\else
  \providecommand{\doi}{doi: \begingroup \urlstyle{rm}\Url}\fi

\bibitem[amani et~al.(2026)amani, Lotfi, Baldwin, Bengio, Farajtabar, Abbe, and
  West]{ICLR2026_c073a9d0}
mohammad~hossein amani, Aryo Lotfi, Nicolas Baldwin, Samy Bengio, Mehrdad
  Farajtabar, Emmanuel Abbe, and Robert West.
\newblock Rl for reasoning by adaptively revealing rationales.
\newblock In C.~Vondrick, B.~Hariharan, C.~Raffel, L.~Pinto, D.~Yang, and
  A.~Faust (eds.), \emph{International Conference on Learning Representations},
  volume 2026, pp.\  118599--118620, 2026.
\newblock URL
  \url{https://proceedings.iclr.cc/paper_files/paper/2026/file/c073a9d002f3dab99a7af18642d74bdb-Paper-Conference.pdf}.

\bibitem[Bae et~al.(2026)Bae, Hong, Lee, Kim, Nam, and Kwak]{bae2026online}
Sanghwan Bae, Jiwoo Hong, Min~Young Lee, Hanbyul Kim, JeongYeon Nam, and
  Donghyun Kwak.
\newblock Online difficulty filtering for reasoning oriented reinforcement
  learning.
\newblock In \emph{Proceedings of the 19th Conference of the European Chapter
  of the Association for Computational Linguistics (Volume 1: Long Papers)},
  pp.\  700--719, 2026.

\bibitem[Balunovic et~al.(2026)Balunovic, Dekoninck, Petrov, Jovanovi{\'c}, and
  Vechev]{balunovic2026matharena}
Mislav Balunovic, Jasper Dekoninck, Ivo Petrov, Nikola Jovanovi{\'c}, and
  Martin Vechev.
\newblock Matharena: Evaluating llms on uncontaminated math competitions.
\newblock \emph{Advances in Neural Information Processing Systems}, 38, 2026.

\bibitem[Beliaev(2026)]{beliaev2026max}
Vladislav Beliaev.
\newblock Max out grpo signal: Adaptive trace prefix control for hard reasoning
  problems.
\newblock \emph{arXiv preprint arXiv:2607.07674}, 2026.

\bibitem[Chen et~al.(2026)Chen, Peng, Choubey, Huang, Zhang, Bansal, and
  Wu]{chen2026nudging}
Justin Chen, Xiangyu Peng, Prafulla~Kumar Choubey, Kung-Hsiang Huang, Jiaxin
  Zhang, Mohit Bansal, and Chien-Sheng Wu.
\newblock Nudging the boundaries of llm reasoning.
\newblock In \emph{International Conference on Learning Representations},
  volume 2026, pp.\  30315--30335, 2026.

\bibitem[Chen et~al.(2025)Chen, Li, Li, Chen, and Lin]{chen2025stepwise}
Peter Chen, Xiaopeng Li, Ziniu Li, Xi~Chen, and Tianyi Lin.
\newblock Stepwise guided policy optimization: Coloring your incorrect
  reasoning in grpo.
\newblock \emph{arXiv preprint arXiv:2505.11595}, 2025.

\bibitem[Davis \& Recht(2025)Davis and Recht]{davis2025objective}
Damek Davis and Benjamin Recht.
\newblock What is the objective of reasoning with reinforcement learning?
\newblock \emph{arXiv preprint arXiv:2510.13651}, 2025.

\bibitem[Dekoninck et~al.(2026)Dekoninck, Jovanovi{\'c}, Gehrunger,
  R{\"o}gnvaldsson, Petrov, Sun, and Vechev]{dekoninck2026beyond}
Jasper Dekoninck, Nikola Jovanovi{\'c}, Tim Gehrunger, K{\'a}ri
  R{\"o}gnvaldsson, Ivo Petrov, Chenhao Sun, and Martin Vechev.
\newblock Beyond benchmarks: Matharena as an evaluation platform for
  mathematics with llms.
\newblock \emph{arXiv preprint arXiv:2605.00674}, 2026.

\bibitem[Dong et~al.(2026)Dong, Jiang, Tao, Liu, Zhang, Mou, Cao, Ma, Chen, Li,
  et~al.]{dong2026rl}
Yihong Dong, Xue Jiang, Yongding Tao, Huanyu Liu, Kechi Zhang, Lili Mou, Rongyu
  Cao, Yingwei Ma, Jue Chen, Binhua Li, et~al.
\newblock Rl-plus: Countering capability boundary collapse of llms in
  reinforcement learning with hybrid-policy optimization.
\newblock In \emph{Proceedings of the 64th Annual Meeting of the Association
  for Computational Linguistics (Volume 1: Long Papers)}, pp.\  43029--43053,
  2026.

\bibitem[Dymetman(2026)]{dymetman2026binary}
Marc Dymetman.
\newblock Binary rewards and reinforcement learning: Fundamental challenges.
\newblock \emph{arXiv preprint arXiv:2605.02375}, 2026.

\bibitem[Fu et~al.(2026)Fu, Chen, Chai, Wang, Tu, Yin, Lin, Zhang, Zhu, and
  Zhao]{fu2026srft}
Yuqian Fu, Tinghong Chen, Jiajun Chai, Xihuai Wang, Songjun Tu, Guojun Yin, Wei
  Lin, Qichao Zhang, Yuanheng Zhu, and Dongbin Zhao.
\newblock Srft: A single-stage method with supervised and reinforcement
  fine-tuning for reasoning.
\newblock In \emph{International Conference on Learning Representations},
  volume 2026, pp.\  102430--102459, 2026.

\bibitem[Guo et~al.(2025)Guo, Yang, Zhang, Song, Wang, Zhu, Xu, Zhang, Ma, Bi,
  et~al.]{guo2025deepseek}
Daya Guo, Dejian Yang, Haowei Zhang, Junxiao Song, Peiyi Wang, Qihao Zhu,
  Runxin Xu, Ruoyu Zhang, Shirong Ma, Xiao Bi, et~al.
\newblock Deepseek-r1: Incentivizing reasoning capability in llms via
  reinforcement learning.
\newblock \emph{arXiv preprint arXiv:2501.12948}, 2025.

\bibitem[Guo et~al.(2026)Guo, Deng, Cheng, and Tang]{guo2026g2rpo}
Yongxin Guo, Wenbo Deng, Zhenglin Cheng, and Xiaoying Tang.
\newblock G2rpo-a: Guided group relative policy optimization with adaptive
  guidance.
\newblock In \emph{Proceedings of the 64th Annual Meeting of the Association
  for Computational Linguistics (Volume 1: Long Papers)}, pp.\  4525--4539,
  2026.

\bibitem[He et~al.(2026)He, Liang, Xu, Liu, Chen, Wang, Song, Yu, Liang, Wang,
  et~al.]{he2026deepmath}
Zhiwei He, Tian Liang, Jiahao Xu, Qiuzhi Liu, Xingyu Chen, Yue Wang, Linfeng
  Song, Dian Yu, Zhenwen Liang, Wenxuan Wang, et~al.
\newblock Deepmath-103k: A large-scale, challenging, decontaminated, and
  verifiable mathematical dataset for advancing reasoning.
\newblock In \emph{International Conference on Learning Representations},
  volume 2026, pp.\  138306--138322, 2026.

\bibitem[Huang et~al.(2025)Huang, Cheng, Qiu, Wang, Xu, Ponti, and
  Titov]{huang2025blending}
Zeyu Huang, Tianhao Cheng, Zihan Qiu, Zili Wang, Yinghui Xu, Edoardo~M Ponti,
  and Ivan Titov.
\newblock Blending supervised and reinforcement fine-tuning with prefix
  sampling.
\newblock \emph{arXiv preprint arXiv:2507.01679}, 2025.

\bibitem[H{\"u}botter et~al.(2026)H{\"u}botter, L{\"u}beck, Behric, Baumann,
  Bagatella, Marta, Hakimi, Shenfeld, Buening, Guestrin,
  et~al.]{hubotter2026reinforcement}
Jonas H{\"u}botter, Frederike L{\"u}beck, Lejs Behric, Anton Baumann, Marco
  Bagatella, Daniel Marta, Ido Hakimi, Idan Shenfeld, Thomas~Kleine Buening,
  Carlos Guestrin, et~al.
\newblock Reinforcement learning via self-distillation.
\newblock \emph{arXiv preprint arXiv:2601.20802}, 2026.

\bibitem[Jaech et~al.(2024)Jaech, Kalai, Lerer, Richardson, El-Kishky, Low,
  Helyar, Madry, Beutel, Carney, et~al.]{jaech2024openai}
Aaron Jaech, Adam Kalai, Adam Lerer, Adam Richardson, Ahmed El-Kishky, Aiden
  Low, Alec Helyar, Aleksander Madry, Alex Beutel, Alex Carney, et~al.
\newblock Openai o1 system card.
\newblock \emph{arXiv preprint arXiv:2412.16720}, 2024.

\bibitem[Korbak et~al.(2022)Korbak, Perez, and Buckley]{korbak2022rl}
Tomasz Korbak, Ethan Perez, and Christopher Buckley.
\newblock Rl with kl penalties is better viewed as bayesian inference.
\newblock In \emph{Findings of the Association for Computational Linguistics:
  EMNLP 2022}, pp.\  1083--1091, 2022.

\bibitem[Lambert et~al.(2024)Lambert, Morrison, Pyatkin, Huang, Ivison,
  Brahman, Miranda, Liu, Dziri, Lyu, et~al.]{lambert2024tulu}
Nathan Lambert, Jacob Morrison, Valentina Pyatkin, Shengyi Huang, Hamish
  Ivison, Faeze Brahman, Lester James~V Miranda, Alisa Liu, Nouha Dziri, Shane
  Lyu, et~al.
\newblock Tulu 3: Pushing frontiers in open language model post-training.
\newblock \emph{arXiv preprint arXiv:2411.15124}, 2024.

\bibitem[Le et~al.(2026)Le, Jeon, Vu, Lai, and Yang]{le2026no}
Thanh-Long~V Le, Myeongho Jeon, Kim Vu, Viet Lai, and Eunho Yang.
\newblock No prompt left behind: Exploiting zero-variance prompts in llm
  reinforcement learning via entropy-guided advantage shaping.
\newblock In \emph{International Conference on Learning Representations},
  volume 2026, pp.\  121956--121982, 2026.

\bibitem[Li et~al.(2026{\natexlab{a}})Li, Lin, Lu, Wen, Yang, Gao, Wu, and
  Zhang]{li2026questa}
Jiazheng Li, Hongzhou Lin, Hong Lu, Kaiyue Wen, Zaiwen Yang, Jiaxuan Gao,
  Yi~Wu, and Jingzhao Zhang.
\newblock Questa: Expanding reasoning capacity in llms via question
  augmentation.
\newblock In \emph{International Conference on Learning Representations},
  volume 2026, pp.\  144191--144215, 2026{\natexlab{a}}.

\bibitem[Li et~al.(2026{\natexlab{b}})Li, Tadiparthi, Lee, Agarwal, Mahjoub,
  Pari, Chen, Zhang, and Leqi]{li2026learning}
Shuozhe Li, Vaishnav Tadiparthi, Kwonjoon Lee, Nakul Agarwal, Hossein~Nourkhiz
  Mahjoub, Ehsan~Moradi Pari, Lizhang Chen, Amy Zhang, and Liu Leqi.
\newblock Learning robust reasoning through guided adversarial self-play.
\newblock \emph{arXiv preprint arXiv:2602.00173}, 2026{\natexlab{b}}.

\bibitem[Li et~al.(2025{\natexlab{a}})Li, Yue, Xu, Jiang, Niu, Lin,
  Ramasubramanian, and Poovendran]{li2025small}
Yuetai Li, Xiang Yue, Zhangchen Xu, Fengqing Jiang, Luyao Niu, Bill~Yuchen Lin,
  Bhaskar Ramasubramanian, and Radha Poovendran.
\newblock Small models struggle to learn from strong reasoners.
\newblock In \emph{Findings of the Association for Computational Linguistics:
  ACL 2025}, pp.\  25366--25394, 2025{\natexlab{a}}.

\bibitem[Li et~al.(2025{\natexlab{b}})Li, Sun, Zhao, Min, Zeng, Wu, Cai, Wang,
  Yin, Chen, et~al.]{li2025staying}
Ziheng Li, Zexu Sun, Jinman Zhao, Erxue Min, Yongcheng Zeng, Hui Wu, Hengyi
  Cai, Shuaiqiang Wang, Dawei Yin, Xu~Chen, et~al.
\newblock Staying in the sweet spot: Responsive reasoning evolution via
  capability-adaptive hint scaffolding.
\newblock \emph{arXiv preprint arXiv:2509.06923}, 2025{\natexlab{b}}.

\bibitem[Liao et~al.(2026)Liao, Dong, Xu, Monz, and Bian]{liao2026self}
Baohao Liao, Hanze Dong, Xinxing Xu, Christof Monz, and Jiang Bian.
\newblock Self-hinting language models enhance reinforcement learning.
\newblock \emph{arXiv preprint arXiv:2602.03143}, 2026.

\bibitem[Liu et~al.(2026{\natexlab{a}})Liu, Farina, and Ozdaglar]{liu2026uft}
Mingyang Liu, Gabriele Farina, and Asuman Ozdaglar.
\newblock Uft: Unifying supervised and reinforcement fine-tuning.
\newblock \emph{Advances in Neural Information Processing Systems},
  38:\penalty0 101347--101383, 2026{\natexlab{a}}.

\bibitem[Liu et~al.(2025{\natexlab{a}})Liu, Li, Cao, Xie, Zhou, Dong, Ma, Han,
  and Zhang]{liu2025superrl}
Yihao Liu, Shuocheng Li, Lang Cao, Yuhang Xie, Mengyu Zhou, Haoyu Dong, Xiaojun
  Ma, Shi Han, and Dongmei Zhang.
\newblock Superrl: Reinforcement learning with supervision to boost language
  model reasoning.
\newblock \emph{arXiv preprint arXiv:2506.01096}, 2025{\natexlab{a}}.

\bibitem[Liu et~al.(2026{\natexlab{b}})Liu, Cao, Chen, and
  Honavar]{liu2026restoring}
Zehao Liu, Yuanpu Cao, Jinghui Chen, and Vasant~G Honavar.
\newblock Restoring the sweet spot: Pass-rate weighted self-distillation for
  llm reasoning.
\newblock \emph{arXiv preprint arXiv:2605.27765}, 2026{\natexlab{b}}.

\bibitem[Liu et~al.(2025{\natexlab{b}})Liu, Gong, Fu, Liu, Chen, Hu, Zhang,
  Liu, Zhang, and Tu]{liu2025ghpo}
Ziru Liu, Cheng Gong, Xinyu Fu, Yaofang Liu, Ran Chen, Shoubo Hu, Suiyun Zhang,
  Rui Liu, Qingfu Zhang, and Dandan Tu.
\newblock Ghpo: Adaptive guidance for stable and efficient llm reinforcement
  learning.
\newblock \emph{arXiv preprint arXiv:2507.10628}, 2025{\natexlab{b}}.

\bibitem[Lv et~al.(2025)Lv, Zuo, Sun, Liu, Wei, Chen, Zhu, Zhang, Wang, Ding,
  et~al.]{lv2025towards}
Xingtai Lv, Yuxin Zuo, Youbang Sun, Hongyi Liu, Yuntian Wei, Zhekai Chen,
  Xuekai Zhu, Kaiyan Zhang, Bingning Wang, Ning Ding, et~al.
\newblock Towards a unified view of large language model post-training.
\newblock \emph{arXiv preprint arXiv:2509.04419}, 2025.

\bibitem[Ma et~al.(2026)Ma, Liang, Qiang, Tang, Ma, Wong, Niu, Shen, He, Li,
  et~al.]{ma2026learning}
Lu~Ma, Hao Liang, Meiyi Qiang, Lexiang Tang, Xiaochen Ma, Zhen Wong, Junbo Niu,
  Chengyu Shen, Runming He, Yanhao Li, et~al.
\newblock Learning what reinforcement learning can't: Interleaved online
  fine-tuning for hardest questions.
\newblock In \emph{International Conference on Learning Representations},
  volume 2026, pp.\  80802--80821, 2026.

\bibitem[Nan et~al.(2025)Nan, Chen, Huang, Lu, Wang, Xie, Xiong, Zeng, Zhou,
  Li, et~al.]{nan2025ngrpo}
Gongrui Nan, Siye Chen, Jing Huang, Mengyu Lu, Dexun Wang, Chunmei Xie, Weiqi
  Xiong, Xianzhou Zeng, Qixuan Zhou, Yadong Li, et~al.
\newblock Ngrpo: Negative-enhanced group relative policy optimization.
\newblock \emph{arXiv preprint arXiv:2509.18851}, 2025.

\bibitem[Nath et~al.(2025)Nath, Lau, Gunjal, Sharma, Baharte, and
  Hendryx]{nath2025adaptive}
Vaskar Nath, Elaine Lau, Anisha Gunjal, Manasi Sharma, Nikhil Baharte, and Sean
  Hendryx.
\newblock Adaptive guidance accelerates reinforcement learning of reasoning
  models.
\newblock \emph{arXiv preprint arXiv:2506.13923}, 2025.

\bibitem[Nie et~al.(2026)Nie, Wu, Liu, Li, Lin, Zijian, Fan, Zheng, and
  Jang]{nie2026gradient}
Wenhua Nie, Jianan Wu, Junlin Liu, Ziwei Li, Zheng Lin, Zhang Zijian, Yilong
  Fan, Haoran Zheng, and Jyh-Shing~Roger Jang.
\newblock Gradient starvation in binary-reward grpo: Why group-mean centering
  fails and why the simplest fix works.
\newblock \emph{arXiv preprint arXiv:2605.07689}, 2026.

\bibitem[Phan et~al.(2023)Phan, Hoffman, Dohan, Douglas, Le, Parisi, Sountsov,
  Sutton, Vikram, and A~Saurous]{phan2023training}
Du~Phan, Matthew~Douglas Hoffman, David Dohan, Sholto Douglas, Tuan~Anh Le,
  Aaron Parisi, Pavel Sountsov, Charles Sutton, Sharad Vikram, and Rif
  A~Saurous.
\newblock Training chain-of-thought via latent-variable inference.
\newblock \emph{Advances in Neural Information Processing Systems},
  36:\penalty0 72819--72841, 2023.

\bibitem[Qu et~al.(2026)Qu, Setlur, Smith, Salakhutdinov, and
  Kumar]{qu2026pope}
Yuxiao Qu, Amrith Setlur, Virginia Smith, Ruslan Salakhutdinov, and Aviral
  Kumar.
\newblock Pope: Learning to reason on hard problems via privileged on-policy
  exploration.
\newblock \emph{arXiv preprint arXiv:2601.18779}, 2026.

\bibitem[Ramasubramanian et~al.(2026)Ramasubramanian, Arora, Tajwar, Zeng, Wu,
  Zhou, Xu, Feng, Song, Singh, et~al.]{ramasubramanian2026tail}
Shrinivas Ramasubramanian, Daman Arora, Fahim Tajwar, Guanning Zeng, Qingyang
  Wu, Zhongzhu Zhou, Chenfeng Xu, Haiwen Feng, Yuda Song, Aarti Singh, et~al.
\newblock Tail-likelihood reinforcement learning.
\newblock \emph{arXiv preprint arXiv:2609.02987}, 2026.

\bibitem[Ren et~al.(2024)Ren, Wu, and Liu]{ren2024learn}
Xuan Ren, Biao Wu, and Lingqiao Liu.
\newblock I learn better if you speak my language: Understanding the superior
  performance of fine-tuning large language models with llm-generated
  responses.
\newblock In \emph{Proceedings of the 2024 Conference on Empirical Methods in
  Natural Language Processing}, pp.\  10225--10245, 2024.

\bibitem[Ren et~al.(2026)Ren, Zhang, Xiao, Zhang, Huang, Qiu, Yu, Chen, and
  Liu]{ren2026recycling}
Yanwei Ren, Haotian Zhang, Likang Xiao, Xikai Zhang, Jiaxing Huang, Jiayan Qiu,
  Baosheng Yu, Quan Chen, and Liu Liu.
\newblock Recycling failures: Salvaging exploration in rlvr via fine-grained
  off-policy guidance.
\newblock \emph{arXiv preprint arXiv:2602.24110}, 2026.

\bibitem[Setlur et~al.(2024)Setlur, Garg, Geng, Garg, Smith, and
  Kumar]{setlur2024rl}
Amrith Setlur, Saurabh Garg, Xinyang Geng, Naman Garg, Virginia Smith, and
  Aviral Kumar.
\newblock Rl on incorrect synthetic data scales the efficiency of llm math
  reasoning by eight-fold.
\newblock \emph{Advances in Neural Information Processing Systems},
  37:\penalty0 43000--43031, 2024.

\bibitem[Setlur et~al.(2026)Setlur, Wang, Cohen, Rashidinejad, and
  Xie]{setlur2026reuse}
Amrith Setlur, Zijian Wang, Andrew Cohen, Paria Rashidinejad, and Sang~Michael
  Xie.
\newblock Reuse your flops: Scaling rl on hard problems by conditioning on very
  off-policy prefixes.
\newblock \emph{arXiv preprint arXiv:2601.18795}, 2026.

\bibitem[Shao et~al.(2024)Shao, Wang, Zhu, Xu, Song, Bi, Zhang, Zhang, Li, Wu,
  et~al.]{shao2024deepseekmath}
Zhihong Shao, Peiyi Wang, Qihao Zhu, Runxin Xu, Junxiao Song, Xiao Bi, Haowei
  Zhang, Mingchuan Zhang, YK~Li, Yang Wu, et~al.
\newblock Deepseekmath: Pushing the limits of mathematical reasoning in open
  language models.
\newblock \emph{arXiv preprint arXiv:2402.03300}, 2024.

\bibitem[Singh et~al.(2023)Singh, Co-Reyes, Agarwal, Anand, Patil, Garcia, Liu,
  Harrison, Lee, Xu, et~al.]{singh2023beyond}
Avi Singh, John~D Co-Reyes, Rishabh Agarwal, Ankesh Anand, Piyush Patil, Xavier
  Garcia, Peter~J Liu, James Harrison, Jaehoon Lee, Kelvin Xu, et~al.
\newblock Beyond human data: Scaling self-training for problem-solving with
  language models.
\newblock \emph{arXiv preprint arXiv:2312.06585}, 2023.

\bibitem[Su et~al.(2026)Su, Guan, Gu, Huang, and Wang]{su2026trust}
Mingyu Su, Jian Guan, Yuxian Gu, Minlie Huang, and Hongning Wang.
\newblock Trust-region adaptive policy optimization.
\newblock In \emph{International Conference on Learning Representations},
  volume 2026, pp.\  51806--51828, 2026.

\bibitem[Tajwar et~al.(2026)Tajwar, Zeng, Zhou, Song, Arora, Jiang, Schneider,
  Salakhutdinov, Feng, and Zanette]{tajwar2026maximum}
Fahim Tajwar, Guanning Zeng, Yueer Zhou, Yuda Song, Daman Arora, Yiding Jiang,
  Jeff Schneider, Ruslan Salakhutdinov, Haiwen Feng, and Andrea Zanette.
\newblock Maximum likelihood reinforcement learning.
\newblock \emph{arXiv preprint arXiv:2602.02710}, 2026.

\bibitem[Team et~al.(2025)Team, Du, Gao, Xing, Jiang, Chen, Li, Xiao, Du, Liao,
  et~al.]{team2025kimi}
Kimi Team, Angang Du, Bofei Gao, Bowei Xing, Changjiu Jiang, Cheng Chen, Cheng
  Li, Chenjun Xiao, Chenzhuang Du, Chonghua Liao, et~al.
\newblock Kimi k1. 5: Scaling reinforcement learning with llms.
\newblock \emph{arXiv preprint arXiv:2501.12599}, 2025.

\bibitem[Xi et~al.(2024)Xi, Chen, Hong, Jin, Zheng, He, Ding, Liu, Guo, Wang,
  et~al.]{xi2024training}
Zhiheng Xi, Wenxiang Chen, Boyang Hong, Senjie Jin, Rui Zheng, Wei He, Yiwen
  Ding, Shichun Liu, Xin Guo, Junzhe Wang, et~al.
\newblock Training large language models for reasoning through reverse
  curriculum reinforcement learning.
\newblock \emph{arXiv preprint arXiv:2402.05808}, 2024.

\bibitem[Xia et~al.(2026)Xia, Xu, Yao, McAuley, and He]{xia2026learning}
Yu~Xia, Canwen Xu, Zhewei Yao, Julian McAuley, and Yuxiong He.
\newblock Learning to hint for reinforcement learning.
\newblock \emph{arXiv preprint arXiv:2604.00698}, 2026.

\bibitem[Yan et~al.(2026)Yan, Li, Hu, Wang, Cui, Qu, Cheng, and
  Zhang]{yan2026learning}
Jianhao Yan, Yafu Li, Zican Hu, Zhi Wang, Ganqu Cui, Xiaoye Qu, Yu~Cheng, and
  Yue Zhang.
\newblock Learning to reason under off-policy guidance.
\newblock \emph{Advances in Neural Information Processing Systems},
  38:\penalty0 117157--117186, 2026.

\bibitem[Yang et~al.(2025)Yang, Li, Yang, Zhang, Hui, Zheng, Yu, Gao, Huang,
  Lv, et~al.]{yang2025qwen3}
An~Yang, Anfeng Li, Baosong Yang, Beichen Zhang, Binyuan Hui, Bo~Zheng, Bowen
  Yu, Chang Gao, Chengen Huang, Chenxu Lv, et~al.
\newblock Qwen3 technical report.
\newblock \emph{arXiv preprint arXiv:2505.09388}, 2025.

\bibitem[Yang et~al.(2026)Yang, Bai, Wu, Yang, Setlur, and Kumar]{yang2026int}
Matthew Yang, Hao Bai, Ian Wu, Gene Yang, Amrith Setlur, and Aviral Kumar.
\newblock Int: Self-proposed interventions enable credit assignment in llm
  reasoning.
\newblock In \emph{International Conference on Learning Representations},
  volume 2026, pp.\  85054--85091, 2026.

\bibitem[Zelikman et~al.(2022)Zelikman, Wu, Mu, and Goodman]{zelikman2022star}
Eric Zelikman, Yuhuai Wu, Jesse Mu, and Noah Goodman.
\newblock Star: Bootstrapping reasoning with reasoning.
\newblock \emph{Advances in Neural Information Processing Systems},
  35:\penalty0 15476--15488, 2022.

\bibitem[Zhang et~al.(2026{\natexlab{a}})Zhang, Lv, Li, Wang, Wang, Hu, and
  Yan]{zhang2026stephint}
Kaiyi Zhang, Ang Lv, Jinpeng Li, Yongbo Wang, Feng Wang, Haoyuan Hu, and Rui
  Yan.
\newblock Stephint: Multi-level stepwise hints enhance reinforcement learning
  to reason.
\newblock In \emph{Proceedings of the 64th Annual Meeting of the Association
  for Computational Linguistics (Volume 1: Long Papers)}, pp.\  37846--37864,
  2026{\natexlab{a}}.

\bibitem[Zhang et~al.(2026{\natexlab{b}})Zhang, Xie, Sun, Chen, Wang, Li, Ding,
  and Zhou]{zhang2026policy}
Wenhao Zhang, Yuexiang Xie, Yuchang Sun, Yanxi Chen, Guoyin Wang, Yaliang Li,
  Bolin Ding, and Jingren Zhou.
\newblock On-policy rl meets off-policy experts: Harmonizing supervised
  fine-tuning and reinforcement learning via dynamic weighting.
\newblock In \emph{International Conference on Learning Representations},
  volume 2026, pp.\  120693--120726, 2026{\natexlab{b}}.

\bibitem[Zhang et~al.(2025)Zhang, Zhang, Sun, Feng, Lu, Yang, and
  Meng]{zhang2025critique}
Xiaoying Zhang, Yipeng Zhang, Hao Sun, Kaituo Feng, Chaochao Lu, Chao Yang, and
  Helen Meng.
\newblock Critique-grpo: Advancing llm reasoning with natural language and
  numerical feedback.
\newblock \emph{arXiv preprint arXiv:2506.03106}, 2025.

\bibitem[Zhang et~al.(2026{\natexlab{c}})Zhang, Wu, Zhu, Tan, Yu, He, and
  Jia]{zhang2026scaf}
Xichen Zhang, Sitong Wu, Yinghao Zhu, Haoru Tan, Shaozuo Yu, Ziyi He, and Jiaya
  Jia.
\newblock Scaf-grpo: Scaffolded group relative policy optimization for
  enhancing llm reasoning.
\newblock In \emph{International Conference on Learning Representations},
  volume 2026, pp.\  131946--131974, 2026{\natexlab{c}}.

\bibitem[Zhang et~al.(2026{\natexlab{d}})Zhang, Huang, Li, Ni, Chen, and
  Oymak]{zhang2026bread}
Xuechen Zhang, Zijian Huang, Yingcong Li, Chenshun Ni, Jiasi Chen, and Samet
  Oymak.
\newblock Bread: Branched rollouts from expert anchors bridge sft \& rl for
  reasoning.
\newblock \emph{Advances in Neural Information Processing Systems},
  38:\penalty0 96726--96752, 2026{\natexlab{d}}.

\bibitem[Zhao et~al.(2026)Zhao, Xie, Liu, Huang, Pang, Chen, and
  Grover]{zhao2026self}
Siyan Zhao, Zhihui Xie, Mengchen Liu, Jing Huang, Guan Pang, Feiyu Chen, and
  Aditya Grover.
\newblock Self-distilled reasoner: On-policy self-distillation for large
  language models.
\newblock \emph{arXiv preprint arXiv:2601.18734}, 2026.

\bibitem[Zhou et~al.(2026{\natexlab{a}})Zhou, Li, Zhang, and
  Leqi]{zhou2026expo}
Ruiyang Zhou, Shuozhe Li, Amy Zhang, and Liu Leqi.
\newblock Expo: Unlocking hard reasoning with self-explanation-guided
  reinforcement learning.
\newblock \emph{Advances in Neural Information Processing Systems},
  38:\penalty0 90989--91016, 2026{\natexlab{a}}.

\bibitem[Zhou et~al.(2026{\natexlab{b}})Zhou, Liu, Wang, Du, Lin, Li, Wang, and
  Pang]{zhou2026variational}
Xiangxin Zhou, Zichen Liu, Haonan Wang, Chao Du, Min Lin, Chongxuan Li, Liang
  Wang, and Tianyu Pang.
\newblock Variational reasoning for language models.
\newblock In \emph{International Conference on Learning Representations},
  volume 2026, pp.\  154370--154407, 2026{\natexlab{b}}.

\bibitem[Zhou et~al.(2026{\natexlab{c}})Zhou, Cao, and Ye]{zhou2026correctness}
Yuhang Zhou, Yixin Cao, and Guangnan Ye.
\newblock From correctness to utility: Gain-based prefix evaluation for llm
  reasoning.
\newblock \emph{arXiv preprint arXiv:2606.07190}, 2026{\natexlab{c}}.

\end{thebibliography}

\clearpage
\appendix
\counterwithin{theorem}{section}

\section{Proofs for the Distance to the Ideal Distribution}
\label{app-proof}

This appendix proves Theorem~\ref{thm-one-attempt} and collects its consequences. Appendix~\ref{app-proof-setup} fixes the setup. Appendix~\ref{app-likelihood} derives the gradient identity in Equation~\eqref{eq-success-score} and the distance of training on the reference. Appendix~\ref{app-proof-theorem} describes the accepted trajectory and the reference surprisal, and then proves Theorem~\ref{thm-one-attempt}. Appendix~\ref{app-consequences} gives an equivalent reading of the closed form. Finally, Appendix~\ref{app-example} illustrates the theorem with a finite example.

\subsection{Setup}
\label{app-proof-setup}

We use the setup of Section~\ref{sec-one-attempt}. Fix a question $x$, its correct reference $r$, and the current model, and suppress the conditioning on $x$. The trajectory space is finite, and $\pi$ denotes the current policy $\pi_{\theta_{\mathrm{old}}}$. Moreover, $C$ is the set of correct trajectories, $p=\pi(C)$ is the success probability, and $\nu=\pi(\cdot\mid C)$ is the ideal distribution in Equation~\eqref{eq-target}. We assume $0<\pi(r)<1$. For a token sequence $h$, we write $\pi(h)$ for the probability that a trajectory sampled from $\pi$ begins with $h$, and $\pi(z\mid h)$ for the probability that the model continues $h$ with $z$. Since the continuation follows the same policy, including its termination rule, autoregressive factorization gives $\pi(h\oplus z)=\pi(h)\pi(z\mid h)$ for every complete trajectory $h\oplus z$.

For a prefix length $0\le t<|r|$, the reference splits as $r=h_t\oplus a_t$, and $Q_t$ is the distribution of the accepted trajectory in Equation~\eqref{eq-acquisition}. Besides $d_t$ and $S_t$ in Equation~\eqref{eq-local-quantities}, the proofs use the probabilities of an exact copy of the reference and of a correct candidate trajectory,
\[
u_t=\pi(a_t\mid h_t),\qquad
v_t=\Pr_{Z\sim\pi(\cdot\mid h_t)}\bigl(h_t\oplus Z\in C\bigr).
\]
Here an exact copy means equality of the complete token sequences, and different token sequences can still follow the same mathematical strategy. By definition, $S_t=-\log u_t$, and in particular $S_0=-\log\pi(r)$ is the reference surprisal of the empty prefix. Since autoregressive factorization gives $u_t=\pi(r)/\pi(h_t)\ge\pi(r)>0$, we have $d_t\le1-u_t<1$, and every $S_t$ is finite.

\subsection{The Ideal Distribution}
\label{app-likelihood}

We first justify the choice of the ideal distribution as the learning target in Section~\ref{sec-preliminaries}. Lemma~\ref{lem-gradient-identity} shows that learning the likelihood of its trajectories increases the success probability in expectation. Lemma~\ref{lem-reference-distance} then gives the distance of training on the reference, which is the first term in Theorem~\ref{thm-one-attempt}.

\begin{lemma}[Gradient identity]
\label{lem-gradient-identity}
Suppose that $\pi_\theta$ is differentiable in $\theta$ and assigns positive probability to every trajectory in a neighborhood of $\theta_{\mathrm{old}}$. Write $p_\theta=\pi_\theta(C)$. Then
\[
\mathbb E_{Y\sim\nu}\Bigl[\nabla_\theta\log\pi_\theta(Y)\big|_{\theta=\theta_{\mathrm{old}}}\Bigr]
=\nabla_\theta\log p_\theta\big|_{\theta=\theta_{\mathrm{old}}}.
\]
\end{lemma}

\begin{proof}
Since the trajectory space is finite, differentiating $p_\theta=\sum_{y\in C}\pi_\theta(y)$ term by term gives
\[
\nabla_\theta\log p_\theta
=\frac{1}{p_\theta}\sum_{y\in C}\nabla_\theta\pi_\theta(y)
=\sum_{y\in C}\frac{\pi_\theta(y)}{p_\theta}\nabla_\theta\log\pi_\theta(y).
\]
At $\theta=\theta_{\mathrm{old}}$, the weight $\pi_\theta(y)/p_\theta$ of each $y\in C$ equals $\nu(y)$ by Equation~\eqref{eq-target}, which proves the claim.
\end{proof}

Here $\nu$ is fixed at the current model, and differentiation applies only to the trajectory log likelihood.

\begin{lemma}[Distance of training on the reference]
\label{lem-reference-distance}
Training on the reference has distance
\[
D_{\mathrm{KL}}(\delta_r\Vert\nu)=\log\frac{p}{\pi(r)}=S_0+\log p\in[0,S_0].
\]
In contrast, if $\nu$ assigns positive probability to a trajectory other than $r$, then $D_{\mathrm{KL}}(\nu\Vert\delta_r)=\infty$.
\end{lemma}

\begin{proof}
Since $\delta_r$ assigns all probability to $r$ and $r\in C$, Equation~\eqref{eq-target} gives $D_{\mathrm{KL}}(\delta_r\Vert\nu)=-\log\nu(r)=\log(p/\pi(r))$. Substituting $S_0=-\log\pi(r)$ gives $S_0+\log p$. This quantity is nonnegative since $\nu(r)\le1$, and it is at most $S_0$ since $p\le1$. For the reverse direction, a trajectory $y\ne r$ with $\nu(y)>0$ has $\delta_r(y)=0$, which makes the term $\nu(y)\log(\nu(y)/\delta_r(y))$ infinite.
\end{proof}

Thus, the direction of KL in Section~\ref{sec-preliminaries} permits a finite comparison with training on the reference.

\subsection{Proof of Theorem~\ref{thm-one-attempt}}
\label{app-proof-theorem}

The closed form in Theorem~\ref{thm-one-attempt} combines two ingredients. Lemma~\ref{lem-accepted-distribution} describes the distribution of the accepted trajectory, and Lemma~\ref{lem-surprisal-monotone} relates the reference surprisal to the probability of the prefix.

\begin{lemma}[Distribution of the accepted trajectory]
\label{lem-accepted-distribution}
For every prefix length $t$, we have $d_t=v_t-u_t$ and $Q_t(r)=1-d_t$. Moreover, every trajectory $y\ne r$ with $Q_t(y)>0$ is correct and has the form $y=h_t\oplus z$, and it satisfies
\[
Q_t(y)=\pi(z\mid h_t),\qquad
\frac{Q_t(y)}{\nu(y)}=\frac{p}{\pi(h_t)}.
\]
These trajectories have total probability $d_t$ under $Q_t$.
\end{lemma}

\begin{proof}
The candidate trajectory $\widetilde Y_t=h_t\oplus Z$ equals $r$ exactly when $Z=a_t$, which has probability $u_t$. Since $r\in C$, the event $\{\widetilde Y_t\in C\}$ is the disjoint union of $\{\widetilde Y_t=r\}$ and $\{\widetilde Y_t\in C,\ \widetilde Y_t\ne r\}$. Therefore, $v_t=u_t+d_t$. By Equation~\eqref{eq-acquisition}, the accepted trajectory equals $r$ when the candidate trajectory is an exact copy or is incorrect. These two events are disjoint and have total probability $u_t+(1-v_t)=1-d_t$.

For $y\ne r$, the accepted trajectory equals $y$ only when $\widetilde Y_t=y$ and $y\in C$. Hence, $Q_t(y)>0$ requires $y\in C$ and $y=h_t\oplus z$ for some $z$, in which case $Q_t(y)=\pi(z\mid h_t)$. Autoregressive factorization gives $\pi(z\mid h_t)=\pi(y)/\pi(h_t)$, where $\pi(h_t)\ge\pi(r)>0$. Since $\nu(y)=\pi(y)/p$ for $y\in C$, the ratio $Q_t(y)/\nu(y)$ equals $p/\pi(h_t)$. Finally, these trajectories carry the remaining probability $1-Q_t(r)=d_t$.
\end{proof}

\begin{lemma}[Monotonicity of the reference surprisal]
\label{lem-surprisal-monotone}
For every prefix length $t$,
\[
S_0-S_t=-\log\pi(h_t),\qquad 0\le S_t\le S_0.
\]
Moreover, $S_t$ is nonincreasing in $t$.
\end{lemma}

\begin{proof}
Autoregressive factorization gives $\pi(r)=\pi(h_t)\pi(a_t\mid h_t)$. Taking negative logarithms gives $S_0=-\log\pi(h_t)+S_t$. Since $\pi(a_t\mid h_t)\le1$ and $\pi(h_t)\le1$, both $S_t$ and $S_0-S_t$ are nonnegative. For $s\le t$, every trajectory that begins with $h_t$ also begins with $h_s$. Therefore, $\pi(h_t)\le\pi(h_s)$, which gives $S_t=S_0+\log\pi(h_t)\le S_s$.
\end{proof}

With these two lemmas, the proof of Theorem~\ref{thm-one-attempt} reduces to a direct computation.

\begin{restated}{thm-one-attempt}{Distance to the ideal distribution}
Under the setup in Appendix~\ref{app-proof-setup}, for every prefix length $t$,
\[
D_{\mathrm{KL}}(Q_t\Vert\nu)
=\underbrace{\log\frac{p}{\pi(r)}}_{D_{\mathrm{KL}}(\delta_r\Vert\nu)}
-d_tS_t+(1-d_t)\log(1-d_t).
\]
Moreover, every $t^*\in\arg\max_t d_tS_t$ satisfies
$D_{\mathrm{KL}}(Q_{t^*}\Vert\nu)\le\min_t D_{\mathrm{KL}}(Q_t\Vert\nu)+1/e$.
\end{restated}

\begin{proof}
Write $d=d_t$. By Lemma~\ref{lem-accepted-distribution}, $Q_t$ assigns probability $1-d$ to $r$, whose likelihood ratio is $Q_t(r)/\nu(r)=(1-d)p/\pi(r)$. It assigns the remaining probability $d$ to trajectories whose likelihood ratio is $p/\pi(h_t)$. Therefore,
\[
D_{\mathrm{KL}}(Q_t\Vert\nu)
=(1-d)\log\frac{(1-d)p}{\pi(r)}+d\log\frac{p}{\pi(h_t)}.
\]
Lemma~\ref{lem-surprisal-monotone} gives $\log\pi(h_t)=\log\pi(r)+S_t$, and hence $d\log(p/\pi(h_t))=d\log(p/\pi(r))-dS_t$. Collecting the two terms in $\log(p/\pi(r))$ proves the closed form. Its first term equals $D_{\mathrm{KL}}(\delta_r\Vert\nu)$ by Lemma~\ref{lem-reference-distance} and does not depend on $t$.

For the second statement, write $D_r=D_{\mathrm{KL}}(\delta_r\Vert\nu)$ and $c_t=(1-d_t)\log(1-d_t)$. Since $-a\log a\in[0,1/e]$ for $a\in[0,1]$, we have $c_t\in[-1/e,0]$. For every prefix length $t$, the maximality of $d_{t^*}S_{t^*}$ gives
\begin{align*}
D_{\mathrm{KL}}(Q_{t^*}\Vert\nu)
&=D_r-d_{t^*}S_{t^*}+c_{t^*}
\le D_r-d_{t^*}S_{t^*}
\le D_r-d_tS_t\\
&=D_{\mathrm{KL}}(Q_t\Vert\nu)-c_t
\le D_{\mathrm{KL}}(Q_t\Vert\nu)+\frac1e.
\end{align*}
Taking the minimum over the finitely many prefix lengths completes the proof.
\end{proof}

Since $d_tS_t\ge0$ and $(1-d_t)\log(1-d_t)\le0$, the closed form also gives $D_{\mathrm{KL}}(Q_t\Vert\nu)\le D_{\mathrm{KL}}(\delta_r\Vert\nu)$ for every prefix length $t$. Therefore, a continuation with fallback never moves the distribution of the accepted trajectory farther from the ideal distribution than training on the reference.

\subsection{Consequences of Theorem~\ref{thm-one-attempt}}
\label{app-consequences}

We now give an equivalent reading of the closed form. Corollary~\ref{cor-remaining-surprisal} rewrites the distance in terms of the reference content kept in the accepted trajectory.

\begin{corollary}[Reference content kept in the accepted trajectory]
\label{cor-remaining-surprisal}
Let $\bar h$ be the reference content kept in the accepted trajectory $Y\sim Q_t$, which is $r$ when $Y=r$ and the prefix $h_t$ otherwise. Then, for every prefix length $t$,
\[
D_{\mathrm{KL}}(Q_t\Vert\nu)
=\log p
+\underbrace{(1-d_t)S_0+d_t(S_0-S_t)}_{\mathbb E[-\log\pi(\bar h)]}
+(1-d_t)\log(1-d_t).
\]
\end{corollary}

\begin{proof}
Lemma~\ref{lem-reference-distance} gives $\log(p/\pi(r))=\log p+S_0$. Substituting this expression and $S_0=(1-d_t)S_0+d_tS_0$ into the closed form of Theorem~\ref{thm-one-attempt} gives the display. By Lemma~\ref{lem-accepted-distribution}, $\bar h=r$ with probability $1-d_t$ and $\bar h=h_t$ with probability $d_t$. Since $-\log\pi(r)=S_0$ and $-\log\pi(h_t)=S_0-S_t$ by Lemma~\ref{lem-surprisal-monotone}, the middle two terms equal $\mathbb E[-\log\pi(\bar h)]$.
\end{proof}

In other words, apart from $\log p$, which does not depend on the prefix, and a term in $[-1/e,0]$, the distance to the ideal distribution is the expected surprisal of the reference content that the model has not replaced. Both extremes of the prefix length keep nearly all of the reference surprisal $S_0$. A prefix covering almost the entire reference keeps it through the prefix, while an empty prefix on a hard question keeps it through fallback.

\subsection{A Finite Example}
\label{app-example}

This example shows that neither the success probability of the continuation nor the successful replacement probability alone determines which prefix gives the smaller distance. Consider five complete trajectories whose displayed strings omit a shared deterministic termination token,
$\mathtt H\oplus\mathtt r$, $\mathtt H\oplus\mathtt s$,
$\mathtt H\oplus\mathtt f$, $\mathtt O\oplus\mathtt t$, and
$\mathtt O\oplus\mathtt g$, with
respective probabilities $0.08$, $0.01$, $0.01$, $0.20$, and $0.70$.
The correct trajectories are $\mathtt H\oplus\mathtt r$,
$\mathtt H\oplus\mathtt s$, and $\mathtt O\oplus\mathtt t$, and the
reference is $\mathtt H\oplus\mathtt r$. The first tokens are
$\mathtt H$ and $\mathtt O$, with probabilities
$0.10$ and $0.90$, and the second-token probabilities are the corresponding
conditional probabilities. These probabilities define a finite autoregressive distribution.

\begin{center}
\begin{tabular}{lcccccc}
\toprule
Prefix & $v_t$ & $u_t$ & $d_t$ & $S_t$ & $d_tS_t$ & $D_{\mathrm{KL}}(Q_t\Vert\nu)$ \\
\midrule
Empty & $0.29$ & $0.08$ & $0.21$ & $2.525729$ & $0.530403$ & $0.571231$ \\
$\mathtt H$ & $0.90$ & $0.80$ & $0.10$ & $0.223144$ & $0.022314$ & $1.170715$ \\
\bottomrule
\end{tabular}
\end{center}

Training on the reference gives $D_{\mathrm{KL}}(\delta_r\Vert\nu)=1.287854$. Supplying $\mathtt H$ increases $v_t$ while making an exact copy more likely. In this setting, both the successful replacement probability $d_t$ and the reference surprisal $S_t$ decrease. Therefore, the empty prefix gives the larger $d_tS_t$ and the smaller distance despite the lower success probability of its continuation.

To isolate the role of the reference surprisal, keep the three probabilities under $\mathtt H$ unchanged and assign probabilities $0.04$ and $0.86$ to $\mathtt O\oplus\mathtt t$ and $\mathtt O\oplus\mathtt g$, respectively. The resulting comparison is
\begin{center}
\begin{tabular}{lcccccc}
\toprule
Prefix & $v_t$ & $u_t$ & $d_t$ & $S_t$ & $d_tS_t$ & $D_{\mathrm{KL}}(Q_t\Vert\nu)$ \\
\midrule
Empty & $0.13$ & $0.08$ & $0.05$ & $2.525729$ & $0.126286$ & $0.310493$ \\
$\mathtt H$ & $0.90$ & $0.80$ & $0.10$ & $0.223144$ & $0.022314$ & $0.368369$ \\
\bottomrule
\end{tabular}
\end{center}
Training on the reference now gives $D_{\mathrm{KL}}(\delta_r\Vert\nu)=0.485508$. Here supplying $\mathtt H$ increases both $v_t$ and $d_t$. However, the decrease in the reference surprisal outweighs the increase in $d_t$. Therefore, the empty prefix still gives the larger $d_tS_t$ and the smaller distance. Thus, even the successful replacement probability alone can favor a different prefix from the one preferred by Theorem~\ref{thm-one-attempt}. Both examples concern a fixed model and illustrate why the choice of prefix should account for $d_t$ and $S_t$ together.

\section{Proofs for Learning the Prefix Selector}
\label{app-selector}

This appendix proves Theorem~\ref{thm-selector} and analyzes the prefix selector during training. Appendix~\ref{app-feedback} fixes the setup and derives the properties of the selector feedback, and Appendix~\ref{app-selector-proof} proves Theorem~\ref{thm-selector}. Appendix~\ref{app-selector-sharing} examines the effects of sharing the selector across questions and of randomizing the guidance level. Finally, Appendix~\ref{app-selector-training} derives the expected update of the selector in the GRPO application as the model and the questions change, and Appendix~\ref{app-selector-noise} bounds the accumulated sampling noise.

\subsection{Setup and Selector Feedback}
\label{app-feedback}

Fix the current model, the references, and $\mathcal D_{\mathrm{guide}}$ as in Section~\ref{sec-shared-selection}. We assume that every question satisfies the setup of Appendix~\ref{app-proof-setup}, which includes $0<\pi_{\theta_{\mathrm{old}}}(r_x\mid x)<1$ and hence $0<S_{0,x}<\infty$. For a question $x$, we write $\nu_x$ for its ideal distribution. For a guidance level $\alpha$, we write $Q_{x,\alpha}$, $d_{x,\alpha}$, and $S_{x,\alpha}$ for the distribution of the accepted trajectory, the successful replacement probability, and the reference surprisal at the prefix length $t_x(\alpha)$. Here $t_x(\alpha)$ is the actual prefix length, including the boundary adjustment in Section~\ref{sec-method-guidance}. Therefore, the nominal level $\alpha$ can differ from the realized ratio $S_{x,\alpha}/S_{0,x}$. We first collect the properties of the selector feedback for a single question, for which we suppress $x$ as in Appendix~\ref{app-proof}.

\begin{lemma}[Selector feedback]
\label{lem-selector-feedback}
For every prefix length $t$, the selector feedback satisfies $W_t\in[0,1]$ and $\mathbb E[W_t]=d_tS_t/S_0$. Moreover, with $c_t=(1-d_t)\log(1-d_t)$,
\begin{equation}
\frac{D_{\mathrm{KL}}(Q_t\Vert\nu)}{S_0}
=\frac{D_{\mathrm{KL}}(\delta_r\Vert\nu)}{S_0}
-\mathbb E[W_t]+\frac{c_t}{S_0},
\qquad
\frac{c_t}{S_0}\in\left[-\frac{1}{eS_0},0\right],
\label{eq-normalized-distance}
\end{equation}
and the normalized distances satisfy $0\le D_{\mathrm{KL}}(Q_t\Vert\nu)/S_0\le D_{\mathrm{KL}}(\delta_r\Vert\nu)/S_0\le1$.
\end{lemma}

\begin{proof}
The indicator in Equation~\eqref{eq-observation} has expectation $d_t$ by Equation~\eqref{eq-local-quantities}. Lemma~\ref{lem-surprisal-monotone} gives $0\le S_t\le S_0$, and hence $W_t\in[0,1]$ and $\mathbb E[W_t]=d_tS_t/S_0$, which proves Equation~\eqref{eq-observation-bound}. Dividing the closed form of Theorem~\ref{thm-one-attempt} by $S_0$ gives Equation~\eqref{eq-normalized-distance}, and the range of $c_t$ follows from $-a\log a\in[0,1/e]$ for $a\in[0,1]$. Finally, since $\mathbb E[W_t]\ge0$ and $c_t\le0$, Equation~\eqref{eq-normalized-distance} gives $D_{\mathrm{KL}}(Q_t\Vert\nu)\le D_{\mathrm{KL}}(\delta_r\Vert\nu)$, and Lemma~\ref{lem-reference-distance} gives $0\le D_{\mathrm{KL}}(\delta_r\Vert\nu)\le S_0$.
\end{proof}

Thus, a larger expected feedback corresponds to a smaller normalized distance, up to the bounded last term. Since normalization divides $d_tS_t$ at every prefix of the same reference by the same $S_0$, it preserves the ranking of $d_tS_t$ across these prefixes. Because of the bounded last term, this ranking need not equal the ranking of the exact distances. Across questions, normalization expresses each distance relative to the reference surprisal of the empty prefix.

The ablation without surprisal weighting in Section~\ref{sec-ablations} replaces $W_t$ with the indicator of successful replacement, whose expectation is $d_t$. Since this indicator still excludes exact copies of the reference, comparing it with $W_t$ isolates the effect of weighting a successful replacement by the reference surprisal.

For the rest of this appendix, we write
\[
g_x(\alpha)=\mathbb E[W_{x,\alpha}]
=\frac{d_{x,\alpha}S_{x,\alpha}}{S_{0,x}}\in[0,1]
\]
for the expected selector feedback of question $x$ at level $\alpha$, where the equality follows from Lemma~\ref{lem-selector-feedback}.

\subsection{Proof of Theorem~\ref{thm-selector}}
\label{app-selector-proof}
\label{app-selector-objective}

Theorem~\ref{thm-selector} measures a selector by its average normalized distance and compares it with training on the reference. Written out, these distances are
\[
\mathcal L(\phi)=\mathbb E_{X\sim\mathcal D_{\mathrm{guide}},\,\alpha\sim\rho_\phi}
\left[\frac{D_{\mathrm{KL}}(Q_{X,\alpha}\Vert\nu_X)}{S_{0,X}}\right],
\qquad
\mathcal L_{\mathrm{ref}}=\mathbb E_{X\sim\mathcal D_{\mathrm{guide}}}
\left[\frac{D_{\mathrm{KL}}(\delta_{r_X}\Vert\nu_X)}{S_{0,X}}\right].
\]
By Lemma~\ref{lem-selector-feedback} and $g_x(\alpha)\in[0,1]$, the quantities $\mathcal L(\phi)$, $\mathcal L_{\mathrm{ref}}$, and $J(\phi)$ lie in $[0,1]$.

\begin{restated}{thm-selector}{Learning the prefix selector from feedback}
Fix the current model, the references, and $\mathcal D_{\mathrm{guide}}$, where every question satisfies the setup of Section~\ref{sec-target}. Let $\mathcal L(\phi)$ and $\mathcal L_{\mathrm{ref}}$ be the expectations of $D_{\mathrm{KL}}(Q_{X,\alpha}\Vert\nu_X)/S_{0,X}$ and $D_{\mathrm{KL}}(\delta_{r_X}\Vert\nu_X)/S_{0,X}$ over $X\sim\mathcal D_{\mathrm{guide}}$ and $\alpha\sim\rho_\phi$. Then
\[
J(\phi)\le\mathcal L_{\mathrm{ref}}-\mathcal L(\phi)\le J(\phi)+\mathbb E_X\left[\frac{1}{eS_{0,X}}\right].
\]
Moreover, for a differentiable prefix selector with positive probabilities and a baseline $b$ fixed before sampling,
\[
\mathbb E\left[(W_{X,\alpha}-b)\nabla_\phi\log\rho_\phi(\alpha)\right]
=\nabla_\phi J(\phi).
\]
\end{restated}

\begin{proof}
For the first statement, averaging Equation~\eqref{eq-normalized-distance} over $X\sim\mathcal D_{\mathrm{guide}}$ and $\alpha\sim\rho_\phi$ gives
\begin{equation}
\mathcal L(\phi)=\mathcal L_{\mathrm{ref}}-J(\phi)+\epsilon(\phi),
\qquad
-\mathbb E_X\frac{1}{eS_{0,X}}\le\epsilon(\phi)\le0,
\label{eq-shared-improvement}
\end{equation}
where $\epsilon(\phi)=\mathbb E[c_{X,\alpha}/S_{0,X}]$ and $c_{x,\alpha}$ is the last term of Theorem~\ref{thm-one-attempt} at the prefix length $t_x(\alpha)$. Lemma~\ref{lem-selector-feedback} bounds each $c_{x,\alpha}/S_{0,x}$, and rearranging Equation~\eqref{eq-shared-improvement} gives the first statement.

For the second statement, the level is sampled independently of the question. Condition on $X=x$. Since $\alpha\sim\rho_\phi$ and $\mathbb E[W_{x,\alpha}\mid\alpha]=g_x(\alpha)$,
\[
\mathbb E\left[(W_{x,\alpha}-b)\nabla_\phi\log\rho_\phi(\alpha)\mid X=x\right]
=\sum_{\alpha\in\mathcal A}(g_x(\alpha)-b)\nabla_\phi\rho_\phi(\alpha)
=\sum_{\alpha\in\mathcal A}g_x(\alpha)\nabla_\phi\rho_\phi(\alpha).
\]
The baseline term vanishes since $\sum_\alpha\nabla_\phi\rho_\phi(\alpha)=\nabla_\phi1=0$. Define $\mu_\alpha=\mathbb E_{X\sim\mathcal D_{\mathrm{guide}}}[g_X(\alpha)]\in[0,1]$, which does not depend on $\phi$. Then
\[
J(\phi)=\sum_{\alpha\in\mathcal A}\rho_\phi(\alpha)\mu_\alpha,
\qquad
\nabla_\phi J(\phi)=\sum_{\alpha\in\mathcal A}\mu_\alpha\nabla_\phi\rho_\phi(\alpha),
\]
where differentiation acts on a finite sum with fixed bounded coefficients. Averaging the conditional identity over $X\sim\mathcal D_{\mathrm{guide}}$ gives the same expression, which proves the second statement.
\end{proof}

The first statement leads to the following guarantee for a selector that maximizes $J$.

\begin{corollary}[Maximizing the objective]
\label{cor-selector-objective}
If $J(\phi^*)\ge J(\phi)$, then $\mathcal L(\phi^*)\le\mathcal L(\phi)+\mathbb E_X[1/(eS_{0,X})]$.
\end{corollary}

\begin{proof}
By Equation~\eqref{eq-shared-improvement}, $J(\phi^*)\ge J(\phi)$ gives
\[
\mathcal L(\phi^*)\le\mathcal L_{\mathrm{ref}}-J(\phi^*)
\le\mathcal L_{\mathrm{ref}}-J(\phi)
\le\mathcal L(\phi)+\mathbb E_X\frac{1}{eS_{0,X}}.
\qedhere
\]
\end{proof}

Thus, a larger objective corresponds to a smaller average normalized distance, up to a bounded correction. Since the correction $\epsilon(\phi)$ can change with the selector, $J$ and $\mathcal L$ can rank selectors differently within this bound. The bound is informative only when $\mathbb E_X[1/S_{0,X}]$ is small, and it can be large when some reference surprisals are close to zero.

In Section~\ref{sec-method-update}, the baseline is a running average of past feedback. For such a baseline, the same argument applies after conditioning on the history and the current question. Given this information, the baseline is fixed, the level still follows $\rho_\phi$, and its feedback still has conditional mean $g_x(\alpha)$. Moreover, the baseline update in Appendix~\ref{app-selector-training} keeps $b\in[0,1]$, which ensures integrability.

\subsection{Sharing and Randomizing Guidance Levels}
\label{app-selector-sharing}

The prefix selector shares one distribution over guidance levels across questions and samples a level for each continuation. We now examine both choices. Proposition~\ref{prop-sharing-cost} bounds the cost of sharing within the objective $J$. Proposition~\ref{prop-selector-mixture} then compares the distance after randomizing the level with the average over levels in $J$.

For any distribution $\rho$ over $\mathcal A$, write $J(\rho)=\sum_\alpha\rho(\alpha)\mu_\alpha$ with $\mu_\alpha$ from the proof of Theorem~\ref{thm-selector}, which gives $J(\rho_\phi)=J(\phi)$. The best shared distribution has value $J_{\mathrm{shared}}^*=\max_\alpha\mu_\alpha$. In contrast, an oracle that knows $g_x$ for every question and chooses a level for each question separately has value $J_{\mathrm{individual}}^*=\mathbb E_X\max_\alpha g_X(\alpha)$.

\begin{proposition}[Cost of sharing]
\label{prop-sharing-cost}
For every distribution $\rho$ over $\mathcal A$,
\begin{equation}
J_{\mathrm{individual}}^*-J(\rho)
=\bigl(J_{\mathrm{individual}}^*-J_{\mathrm{shared}}^*\bigr)
+\bigl(J_{\mathrm{shared}}^*-J(\rho)\bigr),
\label{eq-sharing-cost}
\end{equation}
where both terms are nonnegative. Moreover, the prefix selector in Equation~\eqref{eq-selector} satisfies
\[
\sup_\phi J(\phi)
=(1-\varepsilon_{\mathrm{sel}})\max_\alpha\mu_\alpha
+\frac{\varepsilon_{\mathrm{sel}}}{|\mathcal A|}\sum_\alpha\mu_\alpha.
\]
\end{proposition}

\begin{proof}
The decomposition is an identity. The first term is nonnegative since $\max_\alpha\mathbb E_X[g_X(\alpha)]\le\mathbb E_X\max_\alpha g_X(\alpha)$. The second term is nonnegative since $J(\rho)$ is a convex combination of the values $\mu_\alpha$. For the selector in Equation~\eqref{eq-selector}, write $q_\phi$ for its softmax component. Then $J(\phi)=(1-\varepsilon_{\mathrm{sel}})\sum_\alpha q_\phi(\alpha)\mu_\alpha+(\varepsilon_{\mathrm{sel}}/|\mathcal A|)\sum_\alpha\mu_\alpha$. The first sum is at most $\max_\alpha\mu_\alpha$, and it approaches this value as $q_\phi$ concentrates on a maximizing level.
\end{proof}

The first term in Equation~\eqref{eq-sharing-cost} measures the cost of using the same distribution over levels for every question. It is zero if one level maximizes $g_x$ for almost every question. The second term measures the gap between the current selector and the best shared distribution. If all $\mu_\alpha$ are equal, every shared distribution has the same objective even when individual questions prefer different levels. Moreover, the uniform component of the selector has an explicit cost in the current objective, while it keeps every level observed as the model changes.

Through Theorem~\ref{thm-selector}, the objective $J$ corresponds to the average of the normalized distance over the selected levels. However, sampling a level and then making one continuation produces a single distribution of the accepted trajectory, which mixes the distributions for different levels. For a fixed question $x$ and a distribution $\rho$ over $\mathcal A$, this mixture is $Q_{x,\rho}=\sum_\alpha\rho(\alpha)Q_{x,\alpha}$. Let $I_x$ be the mutual information between the sampled level and the accepted trajectory, and let $\mathcal H(\rho)=-\sum_\alpha\rho(\alpha)\log\rho(\alpha)$ be the entropy of $\rho$.

\begin{proposition}[Randomizing the guidance level]
\label{prop-selector-mixture}
For every question $x$ and distribution $\rho$ over $\mathcal A$,
\[
D_{\mathrm{KL}}(Q_{x,\rho}\Vert\nu_x)
=\sum_\alpha\rho(\alpha)D_{\mathrm{KL}}(Q_{x,\alpha}\Vert\nu_x)-I_x,
\]
where $0\le I_x\le\mathcal H(\rho)$.
Consequently, with $j_x(\rho)=\sum_\alpha\rho(\alpha)g_x(\alpha)$,
\[
\frac{D_{\mathrm{KL}}(\delta_{r_x}\Vert\nu_x)}{S_{0,x}}-j_x(\rho)-\frac{1/e+\mathcal H(\rho)}{S_{0,x}}
\le\frac{D_{\mathrm{KL}}(Q_{x,\rho}\Vert\nu_x)}{S_{0,x}}
\le\frac{D_{\mathrm{KL}}(\delta_{r_x}\Vert\nu_x)}{S_{0,x}}-j_x(\rho).
\]
\end{proposition}

\begin{proof}
Every $Q_{x,\alpha}$ has finite distance by Lemma~\ref{lem-selector-feedback}. Splitting the log probability ratio through $Q_{x,\rho}$ and summing over trajectories with positive probability gives
\begin{align*}
\sum_\alpha\rho(\alpha)D_{\mathrm{KL}}(Q_{x,\alpha}\Vert\nu_x)
&=\sum_{\alpha,y}\rho(\alpha)Q_{x,\alpha}(y)
\left[\log\frac{Q_{x,\alpha}(y)}{Q_{x,\rho}(y)}
+\log\frac{Q_{x,\rho}(y)}{\nu_x(y)}\right]\\
&=I_x+D_{\mathrm{KL}}(Q_{x,\rho}\Vert\nu_x),
\end{align*}
since $\sum_\alpha\rho(\alpha)D_{\mathrm{KL}}(Q_{x,\alpha}\Vert Q_{x,\rho})$ equals the mutual information $I_x$. The mutual information is nonnegative and at most the entropy of the level, which gives $0\le I_x\le\mathcal H(\rho)$. For the bounds, divide the identity by $S_{0,x}$ and apply Equation~\eqref{eq-normalized-distance} to each level. The upper bound uses $I_x\ge0$ and $c_{x,\alpha}\le0$, and the lower bound uses $I_x\le\mathcal H(\rho)$ and $c_{x,\alpha}\ge-1/e$.
\end{proof}

The mutual information is zero when all distributions with positive weight coincide. Otherwise, randomizing the level gives a smaller distance than the average over levels. This construction returns one accepted trajectory from one continuation, and Appendix~\ref{app-candidate-selection} analyzes acceptance after multiple continuations.

\subsection{Selector Updates in the GRPO Application}
\label{app-selector-training}

Theorem~\ref{thm-selector} fixes the model, the references, and $\mathcal D_{\mathrm{guide}}$. We now specify $\mathcal D_{\mathrm{guide}}$ and the baseline in the GRPO application of Section~\ref{sec-method-learning}, and then derive the expected update of the selector for a batch and across training iterations.

\noindent\textbf{Questions selected for reference guidance.}
Fix the current model and a reference $r_x$ for each question. Since the $G$ responses of a group are sampled independently, the group is an all-failure group with probability $f_x=(1-p_{\theta_{\mathrm{old}}}(x))^G$. For $X\sim\mathcal D$, define $\zeta=\mathbb E_{\mathcal D}[f_X]$. When $\zeta>0$, the questions that receive reference guidance follow
\begin{equation}
\mathcal D_{\mathrm{guide}}(x)=\frac{\mathcal D(x)f_x}{\zeta}.
\label{eq-rescue-distribution}
\end{equation}
Given the question, the guidance levels and continuations are sampled independently of the group. Therefore, restricting guidance to questions with an all-failure group changes their relative weights while preserving $\mathbb E[W_{x,\alpha}]=g_x(\alpha)$ for each question and level. If $\zeta=0$, no question receives guidance and the selector update is zero almost surely.

\noindent\textbf{Baseline and update order.}
The selector starts with $\phi=0$ and $b=0$. Every term in Equation~\eqref{eq-selector-update} uses the same selector parameters and baseline, which are fixed before the batch is collected. After a nonempty batch of continuations, the baseline is updated by
\begin{equation}
b\leftarrow(1-\beta)b+\frac{\beta}{|\mathcal U|}\sum_{j\in\mathcal U}W_j,
\label{eq-selector-baseline}
\end{equation}
where $0<\beta\le1$. Since this update is a convex combination of values in $[0,1]$, induction gives $b\in[0,1]$ at every iteration. Centering the updated logits by subtracting their mean leaves the selection probabilities unchanged. An empty batch of continuations leaves $\phi$ and $b$ unchanged. Table~\ref{tab-configuration} specifies the numerical hyperparameters. Reference surprisals, verification outcomes, and model parameters are fixed during the selector update. This update precedes the policy update, and both use the continuations generated by the same current model. All continuations supply feedback, including $K$ zero feedback values when no continuation gives a successful replacement and \method falls back to the reference.

\noindent\textbf{Expected batch update.}
Let $\Psi$ denote the summed direction in Equation~\eqref{eq-selector-update} before multiplication by the step size,
\[
\Psi=\sum_{j\in\mathcal U}(W_j-b)\nabla_\phi\log\rho_\phi(\alpha_j),
\]
with $\Psi=0$ when $\mathcal U$ is empty.

\begin{proposition}[Expected batch update]
\label{prop-batch-update}
Consider a batch of $n$ question positions, each with marginal distribution $\mathcal D$. Suppose that the model, the selector, and the baseline are fixed during collection, and that every all-failure group receives $K$ continuations whose levels and feedback follow the laws above. Then
\begin{equation}
\mathbb E[\Psi]=nK\zeta\nabla_\phi J(\phi),
\label{eq-batch-selector-gradient}
\end{equation}
where the right-hand side is zero when $\zeta=0$.
\end{proposition}

\begin{proof}
Write $\Psi$ as a sum over question positions $i$ and continuations $\ell=1,\ldots,K$ of the terms $\mathbf1\{\text{group }i\text{ is all-failure}\}(W_{i\ell}-b)\nabla_\phi\log\rho_\phi(\alpha_{i\ell})$. Conditional on the question $x$ at position $i$ and on its all-failure event, each term has expectation $\sum_\alpha g_x(\alpha)\nabla_\phi\rho_\phi(\alpha)$ by the proof of Theorem~\ref{thm-selector}. Therefore, position $i$ contributes
\[
\sum_x\mathcal D(x)f_xK\sum_\alpha g_x(\alpha)\nabla_\phi\rho_\phi(\alpha)
=K\zeta\sum_x\mathcal D_{\mathrm{guide}}(x)\sum_\alpha g_x(\alpha)\nabla_\phi\rho_\phi(\alpha)
=K\zeta\nabla_\phi J(\phi),
\]
where the first equality uses Equation~\eqref{eq-rescue-distribution}. Summing over the $n$ positions by linearity proves the claim. When $\zeta=0$, no group is an all-failure group almost surely, and hence $\Psi=0$ almost surely.
\end{proof}

Since the proof uses only the marginal laws, it requires no independence between question positions or between continuations. Holding $J$ fixed, the factors $n$, $K$, and $\zeta$ only scale the expected update at a fixed step size. However, changes in the all-failure probabilities also change the question weights in $J$ and hence the update direction.

\noindent\textbf{Changing models and questions.}
Across training iterations, the model, the question distribution, and the selector all change. Let $(\mathcal F_k)_k$ be the increasing filtration of the training history, where $\mathcal F_k$ contains the history before iteration $k$, including the current model, the selector parameters $\phi_k$, and the baseline. For the question positions $X_{k,1},\ldots,X_{k,n}$ at iteration $k$, define their average conditional marginal distribution by
\[
\mathcal D_k(x)=\frac1n\sum_{i=1}^n\Pr(X_{k,i}=x\mid\mathcal F_k).
\]
With the current all-failure probabilities $f_{k,x}$ and expected feedback $g_{k,x}(\alpha)$, define
\[
\zeta_k=\mathbb E_{X\sim\mathcal D_k}[f_{k,X}],
\qquad
\mathcal D_{\mathrm{guide},k}(x)=\frac{\mathcal D_k(x)f_{k,x}}{\zeta_k},
\]
\[
J_k(\phi)=\mathbb E_{X\sim\mathcal D_{\mathrm{guide},k}}
\sum_\alpha\rho_\phi(\alpha)g_{k,X}(\alpha),
\]
where the last two are defined when $\zeta_k>0$. Let $\Psi_k$ be the summed direction at iteration $k$.

\begin{proposition}[Expected update under changing models and questions]
\label{prop-sequential-update}
Suppose that, conditional on $\mathcal F_k$, the groups and continuations at iteration $k$ follow the laws of Proposition~\ref{prop-batch-update} under the current model, except that position $i$ has the marginal distribution $\Pr(X_{k,i}=\cdot\mid\mathcal F_k)$. Then
\begin{equation}
\mathbb E[\Psi_k\mid\mathcal F_k]
=nK\zeta_k\nabla_\phi J_k(\phi_k),
\label{eq-sequential-selector}
\end{equation}
where the right-hand side is zero when $\zeta_k=0$.
\end{proposition}

\begin{proof}
Apply the proof of Proposition~\ref{prop-batch-update} conditionally on $\mathcal F_k$, with the marginal distribution of position $i$ given by $\Pr(X_{k,i}=\cdot\mid\mathcal F_k)$. Since the proof uses only marginal laws, the positions may have different marginal distributions. Summing their contributions replaces $n\mathcal D$ with $n\mathcal D_k$, which gives Equation~\eqref{eq-sequential-selector}.
\end{proof}

The gradient in Equation~\eqref{eq-sequential-selector} holds the history-dependent model, question distribution, and prefix mapping fixed. Thus, the update follows the current objective in expectation even as training changes these quantities. The setting includes sampling with replacement from a fixed distribution, for which $\mathcal D_k=\mathcal D$, and sampling without replacement from a remaining pool of questions.

\subsection{Accumulated Sampling Noise}
\label{app-selector-noise}

Proposition~\ref{prop-sequential-update} describes the expected update at each iteration. We now bound how far the accumulated updates deviate from these expectations, starting with the size of each update.

\begin{lemma}[Bounded update direction]
\label{lem-bounded-direction}
For the selector in Equation~\eqref{eq-selector} and a baseline in $[0,1]$, every term of $\Psi_k$ has Euclidean norm at most $\sqrt2$. Consequently, $\|\Psi_k/(nK)\|_2\le\sqrt2$.
\end{lemma}

\begin{proof}
Let $q_\phi$ be the softmax component of Equation~\eqref{eq-selector} and $\mathbf e_\alpha$ the coordinate vector for level $\alpha$. Then
\[
\nabla_\phi\log\rho_\phi(\alpha)
=\frac{(1-\varepsilon_{\mathrm{sel}})q_\phi(\alpha)}{\rho_\phi(\alpha)}
\bigl(\mathbf e_\alpha-q_\phi\bigr).
\]
The prefactor is at most one since $\rho_\phi(\alpha)\ge(1-\varepsilon_{\mathrm{sel}})q_\phi(\alpha)$. Moreover, $\|\mathbf e_\alpha-q_\phi\|_2^2=(1-q_\phi(\alpha))^2+\sum_{\alpha'\ne\alpha}q_\phi(\alpha')^2\le2(1-q_\phi(\alpha))^2\le2$. Since $W_j$ and $b$ both lie in $[0,1]$, we have $|W_j-b|\le1$, which bounds each term by $\sqrt2$. Finally, $\Psi_k$ has at most $nK$ terms, including the case of no continuations.
\end{proof}

Define the centered errors
\[
\xi_k=\frac{\Psi_k}{nK}-\zeta_k\nabla_\phi J_k(\phi_k).
\]

\begin{proposition}[Accumulated sampling noise]
\label{prop-selector-noise}
Under the conditions of Proposition~\ref{prop-sequential-update} and a baseline in $[0,1]$, for a deterministic number $N_{\mathrm{iter}}$ of iterations,
\begin{equation}
\mathbb E\left\|\frac1{N_{\mathrm{iter}}}\sum_{k=1}^{N_{\mathrm{iter}}}\xi_k\right\|_2^2
\le\frac{2}{N_{\mathrm{iter}}}.
\label{eq-selector-noise}
\end{equation}
\end{proposition}

\begin{proof}
Proposition~\ref{prop-sequential-update} gives $\mathbb E[\xi_k\mid\mathcal F_k]=0$. The conditional variance identity and Lemma~\ref{lem-bounded-direction} then give
\[
\mathbb E\bigl[\|\xi_k\|_2^2\mid\mathcal F_k\bigr]
=\mathbb E\!\left[\left\|\frac{\Psi_k}{nK}\right\|_2^2\middle|\mathcal F_k\right]
-\left\|\mathbb E\!\left[\frac{\Psi_k}{nK}\middle|\mathcal F_k\right]\right\|_2^2\le2.
\]
For $k<k'$, the error $\xi_k$ is measurable with respect to $\mathcal F_{k'}$. Therefore, $\mathbb E\langle\xi_k,\xi_{k'}\rangle=\mathbb E\langle\xi_k,\mathbb E[\xi_{k'}\mid\mathcal F_{k'}]\rangle=0$. Expanding the squared norm of the sum, the cross terms vanish and each of the $N_{\mathrm{iter}}$ diagonal terms is at most $2$, which proves the bound.
\end{proof}

The baseline update in Equation~\eqref{eq-selector-baseline} satisfies the condition on the baseline. Proposition~\ref{prop-selector-noise} controls the noise around the changing expected directions.

\section{Analysis of Multiple Continuations}
\label{app-multiple-continuations}

Theorem~\ref{thm-one-attempt} concerns a single continuation, whereas \method samples $K$ continuations for each question and accepts candidate trajectories by the rule in Section~\ref{sec-method-guidance}. This appendix analyzes the distribution of the accepted trajectory after multiple continuations. Appendix~\ref{app-budget} fixes one prefix and studies how the distance to the ideal distribution changes with the number of continuations. Appendix~\ref{app-candidate-selection} gives a sufficient condition under which an arbitrary acceptance rule keeps the distance no larger than that of training on the reference.

\subsection{Repeated Continuations from a Fixed Prefix}
\label{app-budget}

Fix a prefix length $t$ and let $H_t$ be the event that a trajectory begins with $h_t$. Since the reference is correct and has positive probability, $0<u_t\le v_t\le1$. For a positive integer $k$, draw $k$ independent continuations from $h_t$. The \emph{first-success rule} returns the first correct candidate trajectory, including an exact copy of $r$, and returns $r$ if every candidate trajectory is incorrect. This rule treats an exact copy as a success. Let $Q_t^{(k)}$ be the resulting distribution of the accepted trajectory, and let $\nu_{h_t}=\pi(\cdot\mid C\cap H_t)$ be the ideal distribution restricted to trajectories that begin with $h_t$.

\begin{lemma}[Distribution after repeated continuations]
\label{lem-budget-distribution}
Let $w_k=1-(1-v_t)^k$ be the probability that at least one candidate trajectory is correct, and define $\lambda_k=w_k/v_t$ and $d_t^{(k)}=d_t\lambda_k$. Then
\begin{equation}
Q_t^{(k)}=(1-w_k)\delta_r+w_k\nu_{h_t},
\label{eq-budget-mixture}
\end{equation}
and
\[
D_{\mathrm{KL}}(Q_t^{(k)}\Vert\nu)
=D_{\mathrm{KL}}(\delta_r\Vert\nu)-d_t^{(k)}S_t+\bigl(1-d_t^{(k)}\bigr)\log\bigl(1-d_t^{(k)}\bigr)+d_t^{(k)}\log\lambda_k.
\]
\end{lemma}

\begin{proof}
A correct trajectory $y=h_t\oplus z$ is returned at continuation $j$ with probability $(1-v_t)^{j-1}\pi(z\mid h_t)$. Summing over $j=1,\ldots,k$ gives $Q_t^{(k)}(y)=\lambda_k\pi(z\mid h_t)$ for every correct $y\ne r$. Since $\nu_{h_t}(y)=\pi(z\mid h_t)/v_t$ for these trajectories, this probability equals $w_k\nu_{h_t}(y)$. The same calculation gives $\lambda_ku_t$ for an exact copy, and fallback adds $(1-v_t)^k$. Therefore, $Q_t^{(k)}(r)=(1-w_k)+w_ku_t/v_t=1-d_t^{(k)}$, which proves Equation~\eqref{eq-budget-mixture}. For the distance, the trajectories other than $r$ have total probability $d_t^{(k)}$ and likelihood ratio $\lambda_kp/\pi(h_t)$. Following the proof of Theorem~\ref{thm-one-attempt} with this ratio gives the closed form.
\end{proof}

Compared with the closed form for one continuation, the additional term $d_t^{(k)}\log\lambda_k$ accounts for the larger probability of each trajectory other than the reference after repeated continuations. For $k=1$, we have $\lambda_1=1$, and the lemma recovers Theorem~\ref{thm-one-attempt}. The mixture form in Equation~\eqref{eq-budget-mixture} shows that additional continuations move probability from the reference to $\nu_{h_t}$, which leads to the following limit.

\begin{proposition}[The limit of a fixed prefix]
\label{prop-budget}
The distance $D_{\mathrm{KL}}(Q_t^{(k)}\Vert\nu)$ is nonincreasing in $k$ and satisfies
\[
-\log\nu(H_t)
\le D_{\mathrm{KL}}(Q_t^{(k)}\Vert\nu)
\le-\log\nu(H_t)+(1-v_t)^k\log\frac{v_t}{u_t}.
\]
Consequently, it converges to $-\log\nu(H_t)$ as $k\to\infty$.
\end{proposition}

\begin{proof}
For $y\in C\cap H_t$, we have $\nu(y)=\nu(H_t)\nu_{h_t}(y)$. Hence, every distribution $Q$ supported on $C\cap H_t$ satisfies $D_{\mathrm{KL}}(Q\Vert\nu)=D_{\mathrm{KL}}(Q\Vert\nu_{h_t})-\log\nu(H_t)$. By Lemma~\ref{lem-budget-distribution}, $Q_t^{(k)}$ is such a distribution. Therefore, it suffices to study $D_{\mathrm{KL}}(Q_t^{(k)}\Vert\nu_{h_t})$, and the lower bound follows from its nonnegativity. For the upper bound, convexity of KL applied to Equation~\eqref{eq-budget-mixture} gives $D_{\mathrm{KL}}(Q_t^{(k)}\Vert\nu_{h_t})\le(1-w_k)D_{\mathrm{KL}}(\delta_r\Vert\nu_{h_t})$, where $D_{\mathrm{KL}}(\delta_r\Vert\nu_{h_t})=-\log\nu_{h_t}(r)=\log(v_t/u_t)$ and $1-w_k=(1-v_t)^k$. For monotonicity, take $\ell>k$. If $w_k=1$, then $Q_t^{(k)}=Q_t^{(\ell)}=\nu_{h_t}$. Otherwise, $a=(1-w_\ell)/(1-w_k)\in[0,1]$ satisfies $Q_t^{(\ell)}=aQ_t^{(k)}+(1-a)\nu_{h_t}$, and convexity gives
\[
D_{\mathrm{KL}}(Q_t^{(\ell)}\Vert\nu_{h_t})
\le aD_{\mathrm{KL}}(Q_t^{(k)}\Vert\nu_{h_t})
\le D_{\mathrm{KL}}(Q_t^{(k)}\Vert\nu_{h_t}).
\]
Finally, the upper bound converges to $-\log\nu(H_t)$ since $v_t\ge u_t>0$.
\end{proof}

For the empty prefix, $H_t$ contains every trajectory. Therefore, the limiting distribution is $\nu$ and the limiting distance is zero. A nonempty prefix has the same limit exactly when $H_t$ contains all correct trajectories with positive probability. Proposition~\ref{prop-budget} orders the numbers of continuations for each fixed prefix. However, comparing different prefixes also requires their respective $u_t$ and $v_t$, which determine both the rate and the limit of the decrease.

\subsection{A Sufficient Condition for General Acceptance Rules}
\label{app-candidate-selection}

To cover the acceptance rule of ARG, which excludes exact copies, mixes several prefixes, prefers some candidate trajectories over others, and may accept several of them, we consider a general procedure. Choose $K\ge1$ reference prefixes $h_{t_1},\ldots,h_{t_K}$ before generation. Conditional on this list, draw a continuation $Z_j\sim\pi(\cdot\mid h_{t_j})$ and form the candidate trajectory $\widetilde Y_j=h_{t_j}\oplus Z_j$ for each $j$. The continuations may depend on each other, provided that each candidate trajectory has this conditional marginal law given the entire list. A procedure returns $r$ or one of the correct candidate trajectories. If several trajectories are accepted, we consider the distribution obtained by selecting one of them with nonnegative weights that sum to one. Let $Q$ denote the resulting distribution of the accepted trajectory, including fallback, and use the convention $0\log0=0$.

\begin{proposition}[Multiple continuations and acceptance]
\label{prop-candidate-selection}
Under the setup above, define $d_Q=1-Q(r)$ and $\Lambda=\sum_{j=1}^K\exp(-S_{t_j})$. Then
\begin{equation}
D_{\mathrm{KL}}(Q\Vert\nu)
\le D_{\mathrm{KL}}(\delta_r\Vert\nu)+d_Q\log\Lambda+(1-d_Q)\log(1-d_Q).
\label{eq-candidate-selection-bound}
\end{equation}
In particular, $\Lambda\le1$ guarantees $D_{\mathrm{KL}}(Q\Vert\nu)\le D_{\mathrm{KL}}(\delta_r\Vert\nu)$. A sufficient condition is $\min_jS_{t_j}\ge\log K$.
\end{proposition}

\begin{proof}
A trajectory $y\ne r$ can be selected only if it appears among the candidate trajectories. Thus,
\[
Q(y)\le\sum_{j=1}^K\Pr(\widetilde Y_j=y)
=\sum_{j=1}^K\mathbf1\{y\text{ begins with }h_{t_j}\}\frac{\pi(y)}{\pi(h_{t_j})}
\le\frac{\pi(y)\Lambda}{\pi(r)},
\]
where the last inequality uses $\pi(r)/\pi(h_{t_j})=\exp(-S_{t_j})$ from Lemma~\ref{lem-surprisal-monotone}. Every $y\ne r$ with $Q(y)>0$ is correct and has $\pi(y)>0$. Hence, $Q(y)/\nu(y)\le p\Lambda/\pi(r)$. Since these trajectories have total probability $d_Q$ and the reference has probability $1-d_Q$ with $Q(r)/\nu(r)=(1-d_Q)p/\pi(r)$,
\[
D_{\mathrm{KL}}(Q\Vert\nu)
\le\log\frac p{\pi(r)}+(1-d_Q)\log(1-d_Q)+d_Q\log\Lambda.
\]
Lemma~\ref{lem-reference-distance} gives $\log(p/\pi(r))=D_{\mathrm{KL}}(\delta_r\Vert\nu)$, which proves Equation~\eqref{eq-candidate-selection-bound}. Since $(1-d_Q)\log(1-d_Q)\le0$, and $d_Q\log\Lambda\le0$ when $\Lambda\le1$, the condition $\Lambda\le1$ gives the comparison with training on the reference. Finally, $\min_jS_{t_j}\ge\log K$ gives $\Lambda\le K\exp(-\log K)=1$.
\end{proof}

The argument permits arbitrary selection among correct candidate trajectories, including the exclusion of exact copies and the priority rule of \method. For one prefix used $K$ times, $\Lambda=K\exp(-S_t)$, and the bound becomes $D_{\mathrm{KL}}(\delta_r\Vert\nu)-d_Q(S_t-\log K)+(1-d_Q)\log(1-d_Q)$. The term $\log K$ accounts for the possible concentration from selecting among several candidate trajectories. For $K=1$ and the acceptance rule in Equation~\eqref{eq-acquisition}, we have $d_Q=d_t$, and the bound becomes the equality in Theorem~\ref{thm-one-attempt}.

In \method, the prefixes come from guidance levels sampled before generation. The following corollary extends the comparison to this case.

\begin{corollary}[Guidance levels sampled before generation]
\label{cor-candidate-levels}
Under the setup of Proposition~\ref{prop-candidate-selection}, suppose that the $K$ guidance levels are sampled before generation from any distribution over $\mathcal A^K$, and that $\min_{\alpha\in\mathcal A}S_{t_x(\alpha)}\ge\log K$. Then the distribution $Q$ of the accepted trajectory, averaged over the sampled levels, satisfies $D_{\mathrm{KL}}(Q\Vert\nu)\le D_{\mathrm{KL}}(\delta_r\Vert\nu)$.
\end{corollary}

\begin{proof}
Conditional on each list of levels, every prefix satisfies $S_{t_j}\ge\log K$. Therefore, Proposition~\ref{prop-candidate-selection} bounds the conditional distance by $D_{\mathrm{KL}}(\delta_r\Vert\nu)$. Since $Q$ is the average of these conditional distributions, convexity of KL in its first argument gives the same bound for $Q$.
\end{proof}

The condition in Corollary~\ref{cor-candidate-levels} uses the actual prefix lengths after the boundary adjustment and is available from the same reference surprisals used for prefix selection.

\section{Analysis of Learning from Inserted Trajectories}
\label{app-positive-learning}

The preceding appendices compare distributions of accepted trajectories at a fixed model. This appendix examines how the inserted trajectories contribute to the policy update in Section~\ref{sec-method-learning}. Appendix~\ref{app-inserted-gradient} derives the gradient of inserted trajectories in Equation~\eqref{eq-acquired-nll}. Appendix~\ref{app-group-learning} derives the expected group update under binary rewards for one inserted trajectory, and Appendix~\ref{app-multiple-learning} extends it to multiple inserted trajectories. Finally, Appendix~\ref{app-prefix-learning} relates learning the complete trajectory and learning only the continuation to their respective success events.

\subsection{The Gradient of Inserted Trajectories}
\label{app-inserted-gradient}

Condition on the realized batch, its masks, and the recomputed group advantages. Let $\mathcal I$ index the inserted trajectories, and let $T>0$ be the total number of active response tokens in the batch loss. For token $j$ of an inserted trajectory $y_i$, let $\ell_{i,j}(\theta)$ be its log probability given the original question and all preceding tokens of the trajectory. Inserted trajectories use the clipped objective of Equation~\eqref{eq-grpo-policy-loss} in Appendix~\ref{app-grpo-loss} with the stopped current probability in the ratio denominator,
\[
\mathcal L_{\mathrm{inserted}}(\theta)
=-\frac1T\sum_{i\in\mathcal I}\sum_{j=1}^{|y_i|}
\min\bigl\{\widetilde\omega_{i,j}(\theta)A_i,
\operatorname{clip}(\widetilde\omega_{i,j}(\theta),0.8,1.3)A_i\bigr\},
\]
\[
\widetilde\omega_{i,j}(\theta)
=\exp\bigl(\ell_{i,j}(\theta)-\operatorname{stopgrad}(\ell_{i,j}(\theta))\bigr),
\]
where $\operatorname{stopgrad}$ preserves a value but assigns zero derivative to it.

\begin{lemma}[Gradient of inserted trajectories]
\label{lem-inserted-gradient}
Suppose that the advantages $A_i$ and the token count $T$ are held fixed during differentiation and that every token of an inserted trajectory is active. Then the gradient of $\mathcal L_{\mathrm{inserted}}$ is given by Equation~\eqref{eq-acquired-nll}.
\end{lemma}

\begin{proof}
By construction, $\widetilde\omega_{i,j}=1$ and $\nabla_\theta\widetilde\omega_{i,j}=\nabla_\theta\ell_{i,j}$. Since the clipping interval contains one in its interior, both branches of the clipped loss have the same derivative at this ratio. Therefore, for any fixed real advantage $A_i$, the loss of token $j$ has derivative $-A_i\nabla_\theta\ell_{i,j}/T$. Summing over all positions, including any supplied prefix and the termination token, gives $-A_i\nabla_\theta\log\pi_\theta(y_i\mid x_i)/T$ for each inserted trajectory. Summing over $i\in\mathcal I$ proves Equation~\eqref{eq-acquired-nll}.
\end{proof}

Lemma~\ref{lem-inserted-gradient} identifies the gradient under the prescribed rule of automatic differentiation. For $A_i>0$, minimizing $\mathcal L_{\mathrm{inserted}}$ gives a positively weighted supervised fine-tuning (SFT) contribution. The advantage specifies its learning weight, while candidate generation and acceptance determine the distribution of the trajectory.

\subsection{The Group Update with One Inserted Trajectory}
\label{app-group-learning}

We now use binary correctness rewards to separate the contributions of correct and incorrect responses. This simplification ignores the format reward, which leaves the advantages unchanged when all responses in a group receive the same format reward. The calculation is at the current model and precedes division by the batch token count. Fix a question $x$ with set of correct trajectories $C$, and suppress the conditioning on $x$ unless shown explicitly. Assume a finite trajectory space, a differentiable policy with positive trajectory probabilities in a neighborhood of $\theta_{\mathrm{old}}$, and $0<p=\pi_{\theta_{\mathrm{old}}}(C)<1$. Draw $G\ge2$ independent responses from $\pi_{\theta_{\mathrm{old}}}$ and condition on all of them failing verification. Then replace one uniformly chosen response by $Y\sim Q$, where $Q$ is supported on $C$, and where $Y$ and the replaced position are independent of the original failures. Both $Q$ and the collected responses are fixed during differentiation. Set
\[
p_\theta=\pi_\theta(C),\qquad
\nu_\theta=\pi_\theta(\cdot\mid C),\qquad
\mu_\theta=\pi_\theta(\cdot\mid C^c),
\]
and write $\nu=\nu_{\theta_{\mathrm{old}}}$ for the ideal distribution at the current model.

At the current model, the ratios of the original responses equal one and clipping is inactive. The inserted trajectory follows Lemma~\ref{lem-inserted-gradient}, and all response tokens are included. We consider the group ascent direction, which is the sum of the reward-dependent gradients before division by the batch token count. With one inserted trajectory, the reward mean is $1/G$ and the sample standard deviation is $s_R=1/\sqrt G$. Therefore, the inserted trajectory has advantage $(G-1)/[G(s_R+\delta)]$, and each failure has advantage $-1/[G(s_R+\delta)]$, where $\delta>0$ is the stabilizer in Equation~\eqref{eq-grpo-advantage}. Define the positive coefficient
\[
c_G=\frac{G-1}{G(1/\sqrt G+\delta)}.
\]
Let $F_1,\ldots,F_{G-1}$ denote the original failures that remain after replacement. The realized group direction is
\begin{equation}
g=c_G\left[
\nabla_\theta\log\pi_\theta(Y\mid x)
-\frac1{G-1}\sum_{j=1}^{G-1}
\nabla_\theta\log\pi_\theta(F_j\mid x)
\right]_{\theta=\theta_{\mathrm{old}}}.
\label{eq-group-learning}
\end{equation}

\begin{proposition}[Inserted trajectories and the group update]
\label{prop-group-learning}
\label{prop-acquired-learning}
Under the setup above, let $\bar g_Q=\mathbb E[g]$ be the expectation over the original failures, the replaced position, and the inserted trajectory. Then
\begin{equation}
\bar g_Q
=c_G\left.\nabla_\theta\left[
\log\frac{p_\theta}{1-p_\theta}
-D_{\mathrm{KL}}(Q\Vert\nu_\theta)
\right]\right|_{\theta=\theta_{\mathrm{old}}}.
\label{eq-expected-group-learning}
\end{equation}
In particular, for $Q=\nu$,
\begin{equation}
\bar g_\nu=\frac{c_G}{p(1-p)}
\left.\nabla_\theta p_\theta\right|_{\theta=\theta_{\mathrm{old}}}.
\label{eq-group-success}
\end{equation}
\end{proposition}

\begin{proof}
Conditioning independent responses on all of them belonging to $C^c$ gives independent draws from $\mu_{\theta_{\mathrm{old}}}$, and uniform replacement preserves this distribution for the remaining failures. Therefore, with $F\sim\mu_{\theta_{\mathrm{old}}}$,
\[
\frac{\bar g_Q}{c_G}
=\left.\left[
\mathbb E_Q\nabla_\theta\log\pi_\theta(Y)
-\mathbb E_{\mu_{\theta_{\mathrm{old}}}}\nabla_\theta\log\pi_\theta(F)
\right]\right|_{\theta=\theta_{\mathrm{old}}}.
\]
Applying the argument of Lemma~\ref{lem-gradient-identity} to the event $C^c$ gives $\mathbb E_{\mu_{\theta_{\mathrm{old}}}}\nabla_\theta\log\pi_\theta(F)=\nabla_\theta\log(1-p_\theta)$ at $\theta_{\mathrm{old}}$. For fixed $Q$, expanding $\log\nu_\theta(y)=\log\pi_\theta(y)-\log p_\theta$ gives
\[
\nabla_\theta D_{\mathrm{KL}}(Q\Vert\nu_\theta)
=-\mathbb E_Q\nabla_\theta\log\pi_\theta(Y)
+\nabla_\theta\log p_\theta.
\]
Substituting both expressions proves Equation~\eqref{eq-expected-group-learning}. When $Q=\nu$, the gradient of the KL term vanishes at $\theta_{\mathrm{old}}$ by Lemma~\ref{lem-gradient-identity}. Since $\nabla_\theta\log[p_\theta/(1-p_\theta)]=\nabla_\theta p_\theta/[p_\theta(1-p_\theta)]$, this gives Equation~\eqref{eq-group-success}.
\end{proof}

The ratio $p_\theta/(1-p_\theta)$ is the odds of passing verification. The first term in Equation~\eqref{eq-expected-group-learning} favors success over failure. The second term records which correct trajectories the distribution $Q$ favors. For training on the reference, $D_{\mathrm{KL}}(\delta_r\Vert\nu_\theta)=-\log\nu_\theta(r)$, and this term favors the particular reference within the set of correct trajectories. In contrast, Equation~\eqref{eq-group-success} shows that inserting trajectories from the ideal distribution gives an expected group update proportional to the gradient of the success probability under the original question.

For reference fallback, $Y=r$ makes the positive term in Equation~\eqref{eq-group-learning} a weighted SFT gradient. For an accepted candidate trajectory, the positive term learns its complete sequence, including the supplied prefix. In both cases, the original failures supply the negative terms in the same group update. This combination of learning from correct trajectories with feedback on the model's own failures connects our update to prior work on integrating supervised and reinforcement fine-tuning \citep{yan2026learning,fu2026srft}.

\subsection{Multiple Inserted Trajectories}
\label{app-multiple-learning}

We keep the preceding setup but insert a random number $N_+$ of correct trajectories with $1\le N_+<G$, where a single reference is inserted when no candidate trajectory is accepted. The inserted trajectories may depend on each other. However, they, their number $N_+$, and the replaced positions, which are uniformly chosen and distinct, are independent of the original failures. Conditional on $N_+=m$, let $\widehat Q$ and $\widehat M$ be the empirical distributions of the inserted trajectories and of the remaining failures. Substituting the advantages $(G-m)/[G(s_m+\delta)]$ and $-m/[G(s_m+\delta)]$ of the inserted trajectories and the remaining failures gives the group direction
\[
g=\kappa_m\left[
\mathbb E_{\widehat Q}\nabla_\theta\log\pi_\theta(Y)
-\mathbb E_{\widehat M}\nabla_\theta\log\pi_\theta(F)
\right]_{\theta=\theta_{\mathrm{old}}},
\]
where
\[
\kappa_m=\frac{m(G-m)}{G(s_m+\delta)},
\qquad
s_m=\sqrt{\frac{m(G-m)}{G(G-1)}}.
\]
Define the average coefficient and the advantage-weighted distribution of inserted trajectories by
\[
\bar\kappa=\mathbb E[\kappa_{N_+}],\qquad
Q_+(y)=\frac{\mathbb E[\kappa_{N_+}\widehat Q(y)]}{\bar\kappa}.
\]
Since $1\le N_+<G$, we have $\bar\kappa>0$, and $Q_+$ is a probability distribution on $C$.

\begin{proposition}[Multiple inserted trajectories]
\label{prop-multiple-learning}
Under the setup above,
\[
\mathbb E[g]
=\bar\kappa\left.\nabla_\theta\left[
\log\frac{p_\theta}{1-p_\theta}
-D_{\mathrm{KL}}(Q_+\Vert\nu_\theta)
\right]\right|_{\theta=\theta_{\mathrm{old}}}.
\]
For one inserted trajectory, $\kappa_1=c_G$ and $Q_+=Q$, which recovers Proposition~\ref{prop-group-learning}.
\end{proposition}

\begin{proof}
Conditional on the inserted trajectories and their number, uniform replacement preserves the failure distribution $\mu_{\theta_{\mathrm{old}}}$ for the remaining failures. Therefore, $\mathbb E[\kappa_{N_+}\widehat M(y)]=\bar\kappa\mu_{\theta_{\mathrm{old}}}(y)$, and taking expectations in the group direction gives
\[
\mathbb E[g]=\bar\kappa\left[
\mathbb E_{Q_+}\nabla_\theta\log\pi_\theta(Y)
-\mathbb E_{\mu_{\theta_{\mathrm{old}}}}\nabla_\theta\log\pi_\theta(F)
\right]_{\theta=\theta_{\mathrm{old}}}.
\]
The proof of Proposition~\ref{prop-group-learning} then applies with $c_G$ replaced by $\bar\kappa$ and $Q$ replaced by $Q_+$. Finally, $s_1=1/\sqrt G$ gives $\kappa_1=c_G$.
\end{proof}

The distribution $Q_+$ weights each realized set of inserted trajectories by its total positive advantage $\kappa_{N_+}$. Therefore, it accounts for both acceptance and the resulting composition of the group.

\subsection{Learning the Prefix and the Continuation}
\label{app-prefix-learning}

The likelihood in Equation~\eqref{eq-acquired-nll} includes every token of the complete accepted trajectory under the original question. For a candidate trajectory from a nonempty prefix, the chain rule gives
\begin{equation}
\log\pi_\theta(h_t\oplus z\mid x)
=\log\pi_\theta(h_t\mid x)
+\log\pi_\theta(z\mid x,h_t).
\label{eq-full-response-learning}
\end{equation}
Hence, learning covers both producing the supplied prefix and completing the trajectory, and reference fallback trains the complete reference through Equation~\eqref{eq-acquired-nll}. We now connect the two terms in Equation~\eqref{eq-full-response-learning} to different success events. Fix a prefix length $t$, let $H_t$ be the event that a trajectory begins with $h_t$, and let $\nu_{h_t}=\pi_{\theta_{\mathrm{old}}}(\cdot\mid C\cap H_t)$ as in Appendix~\ref{app-budget}. This distribution is well defined since the correct reference has positive probability. Let $Q_{Z,t}(z)=\nu_{h_t}(h_t\oplus z)$ be the distribution of its continuations. Here $\pi_\theta(C\mid H_t)$ is the probability that a continuation from $h_t$ gives a correct candidate trajectory, which equals $v_t$ at $\theta_{\mathrm{old}}$.

\begin{proposition}[Learning the prefix and the continuation]
\label{prop-prefix-learning}
Under the assumptions of Lemma~\ref{lem-gradient-identity}, at $\theta=\theta_{\mathrm{old}}$,
\[
\mathbb E_{Y\sim\nu_{h_t}}\nabla_\theta\log\pi_\theta(Y)
=\nabla_\theta\log\pi_\theta(C\cap H_t),
\qquad
\mathbb E_{Z\sim Q_{Z,t}}\nabla_\theta\log\pi_\theta(Z\mid h_t)
=\nabla_\theta\log\pi_\theta(C\mid H_t).
\]
Moreover, the distribution $Q_t^{(k)}$ of the accepted trajectory after $k$ continuations in Lemma~\ref{lem-budget-distribution} satisfies
\[
\mathbb E_{Y\sim Q_t^{(k)}}\nabla_\theta\log\pi_\theta(Y)
=(1-w_k)\nabla_\theta\log\pi_\theta(r)
+w_k\nabla_\theta\log\pi_\theta(C\cap H_t).
\]
\end{proposition}

\begin{proof}
The first identity follows from Lemma~\ref{lem-gradient-identity} with $C$ replaced by $C\cap H_t$. For the second, Equation~\eqref{eq-full-response-learning} gives $\nabla_\theta\log\pi_\theta(h_t\oplus Z)=\nabla_\theta\log\pi_\theta(h_t)+\nabla_\theta\log\pi_\theta(Z\mid h_t)$, and $\pi_\theta(C\cap H_t)=\pi_\theta(h_t)\pi_\theta(C\mid H_t)$. Subtracting $\nabla_\theta\log\pi_\theta(h_t)$ from both sides of the first identity gives the second. Finally, the mixture form in Equation~\eqref{eq-budget-mixture}, whose weights are fixed at $\theta_{\mathrm{old}}$, and the first identity give the last statement.
\end{proof}

Under $\nu_{h_t}$, the expected gradient of the complete trajectory addresses the joint event of producing the prefix and passing verification. In contrast, learning only the continuation addresses verification given an externally supplied prefix. The two coincide for the empty prefix. The joint event covers the correct trajectories that begin with this prefix, while the overall success event $C$ also includes correct trajectories that begin differently. Finally, the mixture identity in Proposition~\ref{prop-prefix-learning} shows how fallback and successful continuations contribute to the expected gradient of the accepted trajectory before any weighting.

\section{Implementation and Evaluation Details}
\label{app-implementation}

We first specify the training procedure, reference boundaries, and policy loss. We then give the training configurations, question prompt, and evaluation protocol. Together, these details connect the single-continuation analysis in Theorem~\ref{thm-one-attempt} to the complete training procedure.

\subsection{Training Procedure}
\label{app-training-procedure}

Training uses 11,068 screened DeepMath examples with matched offline DeepSeek-R1 reference solutions \citep{guo2025deepseek}. We start from DeepMath-103K, keep problems with difficulty at least 5.0, and drop problems whose gold answer is not a single short expression, is a multiple-choice letter, or cannot be verified by our checker when written inside \verb|\boxed{}|. From the remaining problems, we randomly sample 16,000 in proportion to difficulty and let the base Qwen3-4B attempt each one eight times under the training prompt and sampling settings. Problems solved in all eight attempts are removed. We also drop the few problems whose reference solution does not end in the gold answer. Table~\ref{tab-configuration} gives the ARG configuration. Algorithm~\ref{alg-guidance} separates collection under a fixed policy from updates of the selector and policy. All continuations, including failures and candidate trajectories not selected for the policy update, provide selector feedback.

In the algorithm, $\mathcal U_x$ contains $K$ distinct batch-wide continuation indices for question $x$. Candidate quantities are computed for every $j\in\mathcal U_x$. The operation $\operatorname{Accept}$ returns the candidate indices selected by the rule in Section~\ref{sec-method-learning}.

\begin{algorithm}[htbp]
\caption{One ARG training update with GRPO}
\label{alg-guidance}
\begin{algorithmic}[1]
\Statex \textbf{Input} Batch $\mathcal B$ with references, policy $\theta$, selector $\phi$, baseline $b$.
\State $\theta_{\mathrm{old}}\gets\theta$, $\mathcal U\gets\varnothing$
\State $\mathbf y_x=(y_{x,i})_{i=1}^G$, $y_{x,i}\overset{\mathrm{iid}}{\sim}\pi_{\theta_{\mathrm{old}}}(\cdot\mid x)$ for $x\in\mathcal B$
\For{$x\in\mathcal B$ with $\sum_{i=1}^G R_{\mathrm{acc}}(x,y_{x,i})=0$}
    \State $S_t\gets-\log\pi_{\theta_{\mathrm{old}}}(a_t\mid x,h_t)$ for $0\le t<|r|$, $\mathcal J_x\gets\varnothing$
    \If{$0<S_0<\infty$}
        \State $\alpha_j\overset{\mathrm{iid}}{\sim}\rho_\phi$, $t_j\gets t_x(\alpha_j)$
        \State $Z_j\sim\pi_{\theta_{\mathrm{old}}}(\cdot\mid x,h_{t_j})$, $\widetilde Y_j\gets h_{t_j}\oplus Z_j$
        \State $W_j\gets\mathbf1\{\widetilde Y_j\in C_x,\ \widetilde Y_j\ne r\}\,S_{t_j}/S_0$
        \State $\mathcal U\gets\mathcal U\cup\mathcal U_x$, $\mathcal J_x\gets\operatorname{Accept}((\widetilde Y_j,t_j)_{j\in\mathcal U_x},r)$
    \EndIf
    \State Replace rows in $\mathbf y_x$ by $(\widetilde Y_j)_{j\in\mathcal J_x}$, or by $r$ if $\mathcal J_x=\varnothing$.
\EndFor
\State Recompute final rewards $R_i$, then $A_i\gets(R_i-\bar R)/(s_R+\delta)$.
\If{$\mathcal U\ne\varnothing$}
    \State $\phi\gets\phi+\eta_{\mathrm{sel}}\sum_{j\in\mathcal U}(W_j-b)\nabla_\phi\log\rho_\phi(\alpha_j)$
    \State $b\gets(1-\beta)b+\frac{\beta}{|\mathcal U|}\sum_{j\in\mathcal U}W_j$
\EndIf
\State Update $\theta$ using Equation~\eqref{eq-acquired-nll} and the remaining-response loss in Appendix~\ref{app-grpo-loss}.
\end{algorithmic}
\end{algorithm}

\begin{table}[htbp]
\centering
\small
\begin{tabular}{lr}
\toprule
\textbf{Setting} & \textbf{ARG configuration} \\
\midrule
Questions per update & 32 \\
Ordinary responses per question $G$ & 8 \\
Continuations per all-failure group $K$ & 8 \\
Guidance levels $\mathcal A$ & $\{1,1/2,1/4\}$ \\
Total acceptance cap in Algorithm~\ref{alg-guidance} & 4 \\
Maximum accepted trajectories with empty prefixes & 4 \\
Maximum accepted trajectories with nonempty prefixes & 1 \\
Selector exploration $\varepsilon_{\mathrm{sel}}$ & $0.2$ \\
Selector step size $\eta_{\mathrm{sel}}$ & $0.3$ \\
Baseline update rate $\beta$ & $0.1$ \\
Training budget & 200 updates \\
Learning rate & $10^{-6}$ \\
Warmup updates & 10 \\
Weight decay & $0.01$ \\
Gradient norm clipping & 1 \\
Policy optimization epochs per batch & 1 \\
Clipping ratio interval & $[0.8,1.3]$ \\
Policy loss averaging & Batch token mean \\
Total reward & Accuracy plus $0.2$ format \\
Advantage stabilizer $\delta$ & $10^{-6}$ \\
Training temperature & 1 \\
Training top-$p$ and top-$k$ & 1 and no truncation \\
Question length limit & 3072 tokens \\
Trajectory length budget including the prefix & 8192 tokens \\
Additional KL loss & None \\
Entropy loss coefficient & 0 \\
\bottomrule
\end{tabular}
\caption{ARG configuration for both model sizes.}
\label{tab-configuration}
\end{table}

\subsection{Prefix Boundaries}
\label{app-boundaries}

The level $\alpha=1$ always uses the empty prefix. For $\alpha<1$, we take the latest prefix length whose reference surprisal is at least $\alpha S_0$ and, if this length is positive, move it to the first valid sentence or formula boundary at or after it. If no such boundary exists, we use the latest earlier valid boundary, or the empty prefix if there is none. A valid boundary falls on a token boundary, closes all mathematical delimiters, and leaves a nontrivial suffix containing the final boxed answer. The selector feedback uses the reference surprisal $S_t$ at this adjusted prefix length, and the prefix counts toward the maximum trajectory length.

\subsection{Policy Loss}
\label{app-grpo-loss}

Final rewards $R_i$ combine binary verifier accuracy with $0.2$ times a binary format reward indicating an extracted boxed answer. Thus, a correct response receives $1.2$, while an incorrect response receives $0.2$ or zero according to its format. We recompute advantages within each final group as
\begin{equation}
A_i=\frac{R_i-\bar R}{s_R+\delta},
\label{eq-grpo-advantage}
\end{equation}
where $\bar R$ and $s_R$ are the mean and sample standard deviation of the group's final rewards, and $\delta=10^{-6}$ ensures numerical stability. All-equal rewards give exactly zero advantages.

For an ordinary response $y_i$, let $y_{i,j}$ be token $j$ and define
\[
\omega_{i,j}(\theta)
=\frac{\pi_\theta(y_{i,j}\mid x_i,y_{i,<j})}
{\pi_{\theta_{\mathrm{old}}}(y_{i,j}\mid x_i,y_{i,<j})}.
\]
\begin{samepage}
Let $\mathcal O$ index the original responses that remain in their groups, and let $T$ be the total number of active response tokens in the final batch. Their loss contribution uses the asymmetric clipping interval $[0.8,1.3]$,
\begin{equation}
\mathcal L_{\mathrm{original}}(\theta)
=-\frac1T\sum_{i\in\mathcal O}\sum_{j=1}^{|y_i|}
\min\big\{\omega_{i,j}(\theta)A_i,
\operatorname{clip}(\omega_{i,j}(\theta),0.8,1.3)A_i\big\}.
\label{eq-grpo-policy-loss}
\end{equation}
\end{samepage}
Inserted trajectories use the stopped current probability in the ratio denominator, as derived in Appendix~\ref{app-positive-learning}. With unit token weights over each complete inserted trajectory, this gives Equation~\eqref{eq-acquired-nll}. Loss-active tokens include supplied prefixes and termination tokens when present, and both contributions share the same batch token denominator.

\subsection{Training Configurations and Question Prompt}
\label{app-training-configurations}

The GRPO, SAGE, BREAD, LUFFY, and ARG runs share the optimizer and student rollout settings in Table~\ref{tab-baseline-configuration}. The resampling and weighting ablations use the same settings. SFT learns from the offline references. For OPSD, we follow \citet{zhao2026self} in using one student response per question, a training generation limit of 1024 tokens, and temperature $1.1$. Table~\ref{tab-baseline-configuration} specifies the remaining settings used in our experiments. All listed runs use 32 questions per update and one optimization epoch per collected batch. Their data and parameter-initialization seeds are zero, and the learning rate is constant after warmup.

\begin{table}[htbp]
\centering
\small
\setlength{\tabcolsep}{6pt}
\begin{tabular}{lccc}
\toprule
\textbf{Setting} & \textbf{GRPO-based runs} & \textbf{SFT} & \textbf{OPSD} \\
\midrule
Learning rate & $10^{-6}$ & $10^{-6}$ & $5\times10^{-7}$ \\
Warmup updates & 10 & 10 & 0 \\
Weight decay & $0.01$ & $0.01$ & 0 \\
Gradient norm clipping & 1 & 1 & $0.1$ \\
Responses per question & 8 & Reference & 1 \\
Response length limit & 8192 & 8192 & 1024 \\
Generation temperature & 1 & N/A & $1.1$ \\
Generation top-$p$ & 1 & N/A & $0.95$ \\
Generation top-$k$ & Unrestricted & N/A & 20 \\
\bottomrule
\end{tabular}
\caption{Recorded configurations for both model sizes. GRPO-based runs here comprise GRPO, SAGE, BREAD, LUFFY, ARG, resampling, and the weighting ablation. The reference-insertion ablation is specified by its replacement rule in Section~\ref{sec-ablations}.}
\label{tab-baseline-configuration}
\end{table}

The ordinary student question prompt and the evaluation prompt contain a single user message with the question, a blank line, and the instruction below. They use no system message or few-shot example. Auxiliary teacher or reference conditioning is specific to each training method. The Qwen3 chat template is applied with thinking mode disabled, which appends an empty thinking block before response generation. The rendered question prompt is shown schematically below.
\begin{center}
\begin{minipage}{\linewidth}
\footnotesize
\begin{verbatim}
<|im_start|>user
<question>

Please reason step by step, and put your final answer within \boxed{}.
<|im_end|>
<|im_start|>assistant
<think>

</think>

\end{verbatim}
\end{minipage}
\end{center}
Here \texttt{<question>} denotes the question text, and the line break before \texttt{<|im\_end|>} is added for display. For ARG generation, the selected reference prefix supplies the beginning of the candidate trajectory. Evaluation supplies no reference prefix.

\subsection{Checkpoint and Evaluation Records}
\label{app-evaluation-records}

Table~\ref{tab-reported-checkpoints} lists the reported checkpoints. GRPO, SAGE, BREAD, LUFFY, ARG, resampling, and the weighting ablations use the fixed budget of 200 updates. Reference insertion also reports its step-200 checkpoint. For SFT and OPSD, we select the reported checkpoints by comparing mean $\passk$ and $\avgk$ on the five evaluation benchmarks.

\begin{table}[htbp]
\centering
\small
\begin{tabular}{lrr}
\toprule
\textbf{Method} & \textbf{4B checkpoint} & \textbf{8B checkpoint} \\
\midrule
Base & 0 & 0 \\
SFT & 25 & 25 \\
OPSD & 60 & 100 \\
GRPO, SAGE, BREAD, LUFFY, ARG & 200 & 200 \\
Resampling & 200 & N/A \\
w/o weighting & 200 & 200 \\
Reference insertion & 200 & N/A \\
\bottomrule
\end{tabular}
\caption{Reported checkpoint steps for both model sizes. Base uses the released parameters. SFT and OPSD use benchmark-selected checkpoints, and the remaining trained entries report step 200.}
\label{tab-reported-checkpoints}
\end{table}

All 20 reported model and method entries use the same 150 evaluation questions, comprising 30 questions from each of the five benchmarks. The benchmark files, answer extractor, and verifier implementation match across entries. Table~\ref{tab-evaluation-configuration} gives the common decoding settings. Each question contributes 12 scored responses.

\begin{table}[htbp]
\centering
\small
\begin{tabular}{lr}
\toprule
\textbf{Setting} & \textbf{Evaluation value} \\
\midrule
Temperature & 1 \\
Top-$p$ & $0.95$ \\
Top-$k$ & 20 \\
Maximum new tokens & 16,384 \\
Maximum context length & 32,768 \\
Scored responses per question & 12 \\
Random seed & 1234 \\
Parameter precision & bfloat16 \\
Generation engine & vLLM 0.11.0 \\
Thinking mode & Disabled \\
\bottomrule
\end{tabular}
\caption{Evaluation configuration shared by all reported entries.}
\label{tab-evaluation-configuration}
\end{table}

The scorer extracts the final balanced boxed answer. A response without a boxed answer is incorrect. Extracted answers are checked with Math-Verify 0.9.0 and a normalized string comparison. For each question, $\passk$ records whether at least one of its 12 responses is correct, while $\avgk$ is the fraction of those responses that are correct. Both metrics are averaged over questions and reported as percentages. Since the benchmarks have equal sizes, the mean over benchmarks also equals the mean over all 150 questions.

\section{Additional Main Results}
\label{app-main-results}

Table~\ref{tab-main-average-results} reports $\avgk$ for the same checkpoints and the same 12 scored responses per question as Table~\ref{tab-main-results}, following the scoring protocol in Appendix~\ref{app-evaluation-records}.

\begin{table}[htbp]
\centering
\small
\setlength{\tabcolsep}{4pt}
\begin{tabular*}{\linewidth}{@{\hspace{6pt}\extracolsep{\fill}}ll*{5}{c}|c@{\hspace{6pt}}}
\toprule
\textbf{Model} & \textbf{Method} & \textbf{AIME24} & \textbf{AIME25} & \textbf{AIME26} & \textbf{HMMT\_Feb} & \textbf{HMMT\_Nov} & \textbf{Avg.} \\
\midrule
\multirow{8}{*}{Qwen3-4B}
 & Base & 22.78 & 22.22 & 13.33 & 11.94 & 9.17 & 15.89 \\
 & SFT & 19.44 & 20.56 & 16.94 & 9.72 & 9.17 & 15.17 \\
 & GRPO & 47.22 & 42.22 & \underline{46.67} & \underline{26.94} & 29.17 & 38.44 \\
 & OPSD & 22.22 & 18.89 & 20.83 & 14.44 & 13.06 & 17.89 \\
 & SAGE & 44.17 & 39.17 & 40.83 & 25.00 & 28.06 & 35.44 \\
 & BREAD & \textbf{56.39} & \textbf{47.78} & \textbf{47.78} & \textbf{29.17} & \textbf{35.56} & \textbf{43.33} \\
 & LUFFY & 50.28 & 43.61 & 41.67 & 25.00 & \underline{33.06} & 38.72 \\
 & \textbf{\method (ours)} & \underline{52.50} & \underline{46.67} & 46.39 & \underline{26.94} & \underline{33.06} & \underline{41.11} \\
\midrule
\multirow{8}{*}{Qwen3-8B}
 & Base & 25.28 & 21.67 & 18.61 & 9.72 & 9.44 & 16.94 \\
 & SFT & 33.06 & 27.78 & 28.06 & 12.22 & 17.50 & 23.72 \\
 & GRPO & 48.89 & 35.83 & 36.67 & 22.22 & 30.56 & 34.83 \\
 & OPSD & 33.89 & 29.44 & 26.67 & 14.17 & 19.44 & 24.72 \\
 & SAGE & 54.72 & 40.00 & \underline{44.17} & \underline{27.22} & 35.28 & 40.28 \\
 & BREAD & \textbf{60.56} & \underline{42.50} & 43.06 & \textbf{27.78} & \textbf{39.44} & \textbf{42.67} \\
 & LUFFY & 50.83 & 42.22 & 43.89 & 25.28 & 31.94 & 38.83 \\
 & \textbf{\method (ours)} & \underline{55.28} & \textbf{46.94} & \textbf{46.39} & \underline{27.22} & \underline{35.56} & \underline{42.28} \\
\bottomrule
\end{tabular*}
\caption{Main experimental results in $\avgk$ percentages. Avg. is the mean over the five benchmarks. Within each model, \textbf{bold} and \underline{underline} mark the best and second-best values in each column.}
\label{tab-main-average-results}
\end{table}

\method improves aggregate $\avgk$ over GRPO by $2.67$ and $7.45$ percentage points on Qwen3-4B and Qwen3-8B, respectively, and achieves the second-highest aggregate $\avgk$ on both models. The difference between the two metrics is consistent with where \method intervenes. Specifically, \method modifies the GRPO update only for all-failure groups, where all $G$ responses to a training question are incorrect, and leaves the update for all other questions unchanged. Therefore, the additional learning signal concentrates on questions that the model currently fails to solve. When a question changes from unsolved to solved by a few of the 12 responses, $\passk$ counts the entire question, whereas $\avgk$ counts only the fraction of correct responses. As a result, the effect of \method appears most clearly in $\passk$, which measures the fraction of questions that the model solves within the sampling budget.

\section{Additional Ablation Results}
\label{app-ablations}

We provide benchmark-level results, then examine how the ablations change the accepted trajectories and the prefix selector during training. Reference insertion and resampling use Qwen3-4B, while the surprisal-weighting comparison covers both model sizes.

\subsection{Evaluation Results}
\label{app-ablation-results}

Tables~\ref{tab-ablation-benchmarks} and~\ref{tab-ablation-benchmarks-average} report $\passk$ and $\avgk$ separately for all five benchmarks. All entries use the checkpoint after 200 training steps and the same evaluation protocol as the main results. Each question contributes 12 scored responses, as specified in Appendix~\ref{app-evaluation-records}.

\begin{table}[htbp]
\centering
\small
\setlength{\tabcolsep}{3pt}
\begin{tabular*}{\linewidth}{@{\hspace{6pt}\extracolsep{\fill}}ll*{5}{c}|c@{\hspace{6pt}}}
\toprule
\textbf{Model} & \textbf{Method} & \textbf{AIME24} & \textbf{AIME25} & \textbf{AIME26} & \textbf{HMMT\_Feb} & \textbf{HMMT\_Nov} & \textbf{Avg.} \\
\midrule
\multirow{4}{*}{Qwen3-4B}
& Reference insertion & 70.00 & 60.00 & 63.33 & 43.33 & 46.67 & 56.67 \\
& Resampling & 76.67 & 63.33 & 63.33 & 43.33 & 53.33 & 60.00 \\
& w/o weighting & 76.67 & 70.00 & 66.67 & 50.00 & 53.33 & 63.33 \\
& \method & 80.00 & 73.33 & 73.33 & 43.33 & 66.67 & 67.33 \\
\midrule
\multirow{2}{*}{Qwen3-8B}
& w/o weighting & 83.33 & 70.00 & 66.67 & 50.00 & 50.00 & 64.00 \\
& \method & 83.33 & 73.33 & 73.33 & 56.67 & 63.33 & 70.00 \\
\bottomrule
\end{tabular*}
\caption{Benchmark-level ablation results in $\passk$ percentages. Avg. is the mean over benchmarks. Here w/o weighting denotes \method without surprisal weighting.}
\label{tab-ablation-benchmarks}
\end{table}

\begin{table}[htbp]
\centering
\small
\setlength{\tabcolsep}{3pt}
\begin{tabular*}{\linewidth}{@{\hspace{6pt}\extracolsep{\fill}}ll*{5}{c}|c@{\hspace{6pt}}}
\toprule
\textbf{Model} & \textbf{Method} & \textbf{AIME24} & \textbf{AIME25} & \textbf{AIME26} & \textbf{HMMT\_Feb} & \textbf{HMMT\_Nov} & \textbf{Avg.} \\
\midrule
\multirow{4}{*}{Qwen3-4B}
& Reference insertion & 36.67 & 34.17 & 34.17 & 17.78 & 19.44 & 28.44 \\
& Resampling & 44.44 & 35.56 & 40.00 & 21.67 & 28.33 & 34.00 \\
& w/o weighting & 43.61 & 40.83 & 39.17 & 21.39 & 28.33 & 34.67 \\
& \method & 52.50 & 46.67 & 46.39 & 26.94 & 33.06 & 41.11 \\
\midrule
\multirow{2}{*}{Qwen3-8B}
& w/o weighting & 45.28 & 33.89 & 38.33 & 18.33 & 23.61 & 31.89 \\
& \method & 55.28 & 46.94 & 46.39 & 27.22 & 35.56 & 42.28 \\
\bottomrule
\end{tabular*}
\caption{Benchmark-level ablation results in $\avgk$ percentages, computed from the same scored responses as Table~\ref{tab-ablation-benchmarks}. Avg. is the mean over benchmarks.}
\label{tab-ablation-benchmarks-average}
\end{table}

For the ablations in Table~\ref{tab-ablation-benchmarks-average}, \method has higher average sampled accuracy on every benchmark.

\subsection{Accepted Candidate Trajectories during Training}
\label{app-ablation-generation}

To measure how often continuations supply the trajectories used for training, we count all-failure groups that accept at least one candidate trajectory. Reference fallback is excluded from this count. Figure~\ref{fig-ablation-generation} reports the proportion in consecutive windows of 50 training steps, using the total group counts within each window. Each curve describes the all-failure questions encountered in its own training run.

\begin{figure}[htbp]
\centering
\captionsetup[subfigure]{labelformat=simple,labelsep=space}
\begin{subfigure}{0.48\linewidth}
    \includegraphics[width=\linewidth]{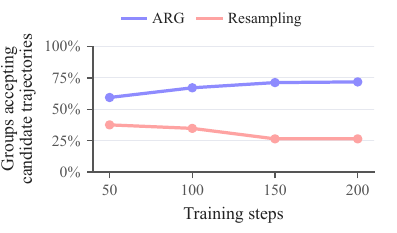}
    \caption{Qwen3-4B resampling}
\end{subfigure}\hfill
\begin{subfigure}{0.48\linewidth}
    \includegraphics[width=\linewidth]{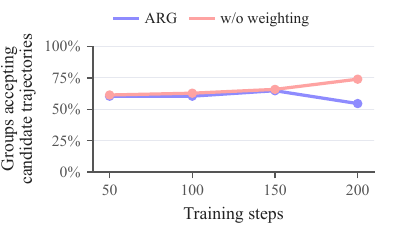}
    \caption{Qwen3-8B surprisal weighting}
\end{subfigure}
\caption{Proportion of all-failure groups that accept candidate trajectories. A point at step $s$ aggregates steps $s-49$ through $s$. Reference fallback is excluded. The curves come from their respective training runs.}
\label{fig-ablation-generation}
\end{figure}

On Qwen3-4B, the proportion increases from $59.4\%$ to $71.8\%$ for \method across the first and last windows, while resampling decreases from $37.6\%$ to $26.4\%$. Over the full runs, \method accepts candidate trajectories in 318 of 482 all-failure groups and uses reference fallback in the remaining 164. Resampling accepts candidate trajectories in 145 of 446 all-failure groups, leaving the remaining 301 unchanged. Resampling retains all correct candidate trajectories. Both methods use eight continuations per all-failure group, and \method also uses reference guidance and fallback.

The weighting ablation gives a different comparison. On Qwen3-8B, it accepts candidate trajectories in $64.8\%$ of all-failure groups, compared with $60.2\%$ for \method, despite its lower evaluation scores. On Qwen3-4B, these proportions are similar at $65.9\%$ and $66.0\%$. Thus, the proportion of groups that accept candidate trajectories alone does not explain the evaluation results.

\subsection{Surprisal Weighting and Prefix Selection}
\label{app-ablation-weighting}

The weighting ablation uses the successful replacement indicator as feedback, while retaining the same guidance levels, prefix construction, selector hyperparameters, and acceptance rule. Figure~\ref{fig-ablation-selectors} compares the resulting selection probabilities. On Qwen3-4B, both selectors favor $\alpha=1/4$ for most of training. On Qwen3-8B, the unweighted selector favors $\alpha=1/4$ throughout training, whereas the probabilities under \method change more over time. The final probabilities under \method are $49.0\%$ for $\alpha=1/2$ and $34.6\%$ for $\alpha=1/4$, compared with $6.8\%$ and $86.4\%$ without weighting.

\begin{figure}[!htbp]
\centering
\captionsetup[subfigure]{labelformat=simple,labelsep=space}
\includegraphics[width=\linewidth]{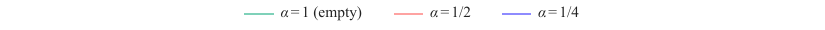}\par
\begin{subfigure}{0.48\linewidth}
    \includegraphics[width=\linewidth]{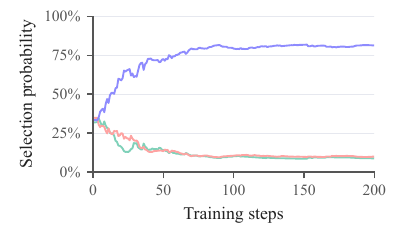}
    \caption{Qwen3-4B \method}
\end{subfigure}\hfill
\begin{subfigure}{0.48\linewidth}
    \includegraphics[width=\linewidth]{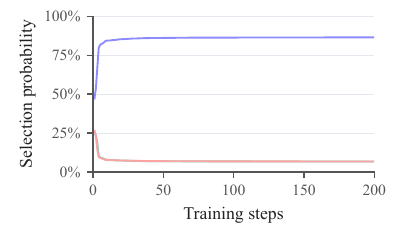}
    \caption{Qwen3-4B w/o weighting}
\end{subfigure}\\[0.5em]
\begin{subfigure}{0.48\linewidth}
    \includegraphics[width=\linewidth]{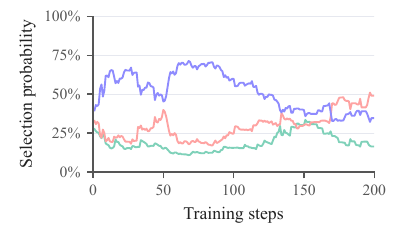}
    \caption{Qwen3-8B \method}
\end{subfigure}\hfill
\begin{subfigure}{0.48\linewidth}
    \includegraphics[width=\linewidth]{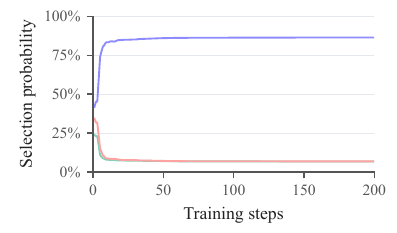}
    \caption{Qwen3-8B w/o weighting}
\end{subfigure}
\caption{Selection probabilities of the prefix selector with and without surprisal weighting. Curves show the state after each training step without smoothing. Steps with no selector update retain the previous probabilities. The level $\alpha=1$ uses an empty prefix, while smaller levels leave less of the reference surprisal to the continuation.}
\label{fig-ablation-selectors}
\end{figure}

\section{Training Dynamics}
\label{app-training-dynamics}

We supplement Section~\ref{sec-training-dynamics} with the corresponding analysis on Qwen3-4B and the sources of accepted trajectories for both models. Comparisons with the ablated methods appear in Appendix~\ref{app-ablations}.

\subsection{Successful Replacement, Feedback, and Prefix Selection}
\label{app-dynamics-feedback}

Figure~\ref{fig-dynamics-feedback-4b} reports successful replacement rates, mean selector feedback, and prefix selection on Qwen3-4B. As in the Qwen3-8B plots, each point in the first two panels uses all continuations sampled at that level within a window of 50 training steps. Successful replacement requires a correct candidate trajectory different from the reference. We compute $W_t$ with the actual $S_t/S_0$ after prefix-boundary adjustment, including zero feedback from failures and exact reference copies. The third panel shows the selection probabilities after each training step without smoothing.

\begin{figure}[htbp]
    \centering
    \captionsetup[subfigure]{labelformat=simple,labelsep=space}
    \includegraphics[width=\linewidth]{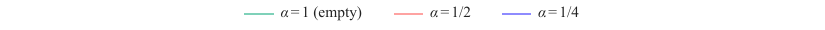}\par
    \begin{subfigure}{0.32\linewidth}
        \includegraphics[width=\linewidth]{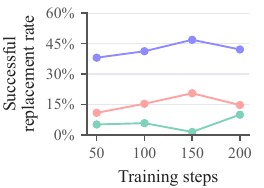}
        \caption{Successful replacement}
    \end{subfigure}\hfill
    \begin{subfigure}{0.32\linewidth}
        \includegraphics[width=\linewidth]{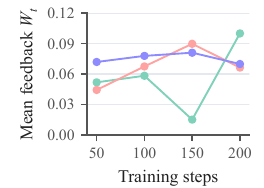}
        \caption{Selector feedback}
    \end{subfigure}\hfill
    \begin{subfigure}{0.32\linewidth}
        \includegraphics[width=\linewidth]{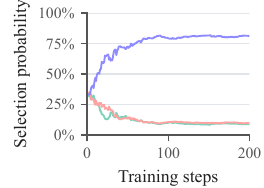}
        \caption{Prefix selection}
        \label{fig-dynamics-selector-4b}
    \end{subfigure}
    \caption{Training dynamics of a Qwen3-4B \method run. In panels (a) and (b), a point at step $s$ averages the continuations at each level from steps $s-49$ through $s$. Panel (c) shows the selection probabilities after each training step without smoothing. The colors and corresponding axis ranges match Figure~\ref{fig-guidance-dynamics}.}
    \label{fig-dynamics-feedback-4b}
    \vspace{-10pt}
\end{figure}

The level $\alpha=1/4$ has the highest successful replacement rate in each window. Weighting reduces its advantage over $\alpha=1/2$, and their mean-feedback ordering varies across windows. Over the full run, the mean feedback is $0.0754$ for $\alpha=1/4$ and $0.0588$ for $\alpha=1/2$. These observations accompany the persistent preference for $\alpha=1/4$ shown in Figure~\ref{fig-dynamics-selector-4b}.

The windowed statistics summarize continuations collected as the policy, the all-failure questions, and the selection probabilities change. Since the selector updates sequentially from these continuations, its state reflects their order and the preceding updates. Steps with no selector update retain the previous probabilities. This gives a view of the feedback observed during training alongside the resulting selection history.

\subsection{Sources of Accepted Trajectories}
\label{app-dynamics-sources}

Figure~\ref{fig-dynamics-response-sources} records which trajectories enter policy learning. We call a candidate trajectory guided if it comes from a nonempty prefix and unguided if it comes from the empty prefix. Each all-failure group belongs to one of four categories, according to whether it accepts only guided candidate trajectories, only unguided ones, both types, or the reference as fallback. Thus, the figure describes groups, with each group counted once even when it accepts multiple trajectories.

\begin{figure}[htbp]
    \centering
    \captionsetup[subfigure]{labelformat=simple,labelsep=space}
    \includegraphics[width=\linewidth]{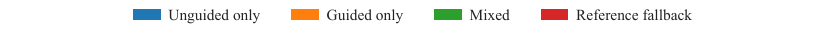}\par
    \begin{subfigure}{0.48\linewidth}
        \includegraphics[width=\linewidth]{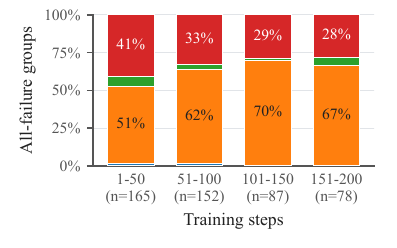}
        \caption{Qwen3-4B}
    \end{subfigure}\hfill
    \begin{subfigure}{0.48\linewidth}
        \includegraphics[width=\linewidth]{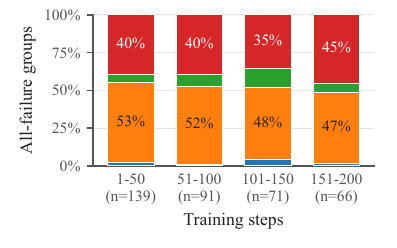}
        \caption{Qwen3-8B}
    \end{subfigure}
    \caption{Sources of accepted trajectories among all-failure groups in consecutive windows of 50 training steps. The mixed category contains groups that accept both guided and unguided candidate trajectories. Fallback supplies the complete reference when no candidate trajectory is accepted. Counts below the windows give the number of all-failure groups.}
    \label{fig-dynamics-response-sources}
\end{figure}

Candidate trajectories are accepted in $66.0\%$ and $60.2\%$ of all-failure groups for Qwen3-4B and Qwen3-8B, respectively. The remaining groups use reference fallback. Most groups that accept candidate trajectories use guided ones, while a smaller fraction accept unguided ones. Fallback remains present throughout training, and its proportion varies across windows as the all-failure questions change. These records complement the selection curves by showing the trajectories used for the policy updates.

\end{document}